\documentclass{article}

\usepackage{preprint,times}

\usepackage[utf8]{inputenc}
\usepackage[T1]{fontenc}
\usepackage{amsmath,amssymb,amsthm,mathtools}
\usepackage{booktabs,longtable,array}
\usepackage{multirow}
\usepackage{graphicx}
\usepackage{subcaption}
\usepackage{hyperref}
\usepackage{titletoc} 
\usepackage{algorithm,algorithmic} 
\usepackage{url}
\usepackage{xcolor}
\usepackage{colortbl} 

\newtheorem{theorem}{Theorem}
\newtheorem{lemma}{Lemma}
\newtheorem{proposition}{Proposition}
\newtheorem{corollary}{Corollary}

\newtheorem{assumption}{Assumption}
\theoremstyle{remark}

\newtheorem{example}{Example}

\newcommand{\cM}{\mathcal{M}}

\newcommand{\E}{\mathbb{E}}
\newcommand{\Prob}{\mathbb{P}}
\newcommand{\ind}{\mathbf{1}}

\title{CopRoute: Dependence-Aware Conformal Routing for Language Models}

\author{Siva Rajesh Kasa\thanks{Equal contribution.}\hspace{0.55em}\thanks{This work does not relate to the author's position at Amazon.}\hspace{0.55em}\thanks{Corresponding author.} \\
\texttt{kasa@u.nus.edu}
\And
Santhosh Kumar Kasa\footnotemark[1]\hspace{0.55em}\footnotemark[2]
\And
Karan Gupta
}

\preprintauthors 

\begin{document}
\maketitle

\begin{abstract}

Set-based language model routing aims to increase the coverage of the event where at least one model answers a query correctly while controlling the number of inference calls made. Existing risk-controlled routing policies bound the all-miss risk (the probability that every model in the selected set fails) by calibrating a single conformal threshold and applying it to a nested family of model sets. Because each model is scored in isolation, these policies ignore the dependence structure of model failures. Models that share training data or architecture tend to fail on the same queries, so two individually strong models can be jointly redundant while a weaker but complementary model covers the queries that the others miss. To account for this dependence, which existing risk-controlled policies ignore, we propose CopRoute, which orders models by how many new queries each one answers correctly beyond those already selected, decides per query how far down that order to go, and then picks the smallest cutoff that still meets the target risk. We show that CopRoute's ordering attains, on the design data, the best guarantee any efficient fixed-ordering algorithm can give for the average number of calls to the first correct model, and that a separate calibration step certifies the all-miss risk of the deployed sets. On RouterBench and RouterEval, CopRoute attains the fewest certified calls on nine of ten datasets at a target risk of 0.05 and on eight of ten at 0.10, making up to 79 percent fewer calls (median 24 percent) than the state-of-the-art risk-controlled policy at the same certified all-miss risk.

\end{abstract}

\section{Introduction}
\label{sec:intro}

A deployed language-model (LM) system typically consists of a pool of models that differ in size, task expertise, architecture, and other attributes \citep{hu2024routerbench,huang2025routereval,varangotreille2025survey}. A given model in this pool is usually right on some queries and wrong on others, and no single model dominates across all tasks \citep{srivatsa2024routing,chen2024routerdc,zhuang2025embedllm}. An LM system aims to obtain the right answer for a given input with the minimum number of inference calls \citep{chen2023frugalgpt,ong2024routellm,ding2024hybridllm,ding2025bestroute}. This goal has been posed as a statistical optimization problem based on the risk-controlled routing framework \citep{racer2026,su2025cprouter,angelopoulos2022crc}. The risk-controlled routing objective is to minimize the number of inference calls across the pool of models, and the risk is the probability of not obtaining at least one correct answer from the called models.

Existing methods for this problem fall into two categories. The first category does not use the dependence or co-failure information among the candidate models. Single-model routers choose one model from estimates of its quality and cost \citep{chen2023frugalgpt,ong2024routellm,ding2024hybridllm,chen2024routerdc,feng2025graphrouter,li2025llmbandit,wang2025barp}, while cascades move through models until a stopping rule is met \citep{dohan2022cascades,gupta2024cascades,ding2025bestroute}. RACER instead returns a set with a statistical risk guarantee \citep{racer2026}. It applies one calibrated threshold to per-model scores, and relaxing this threshold produces larger nested model sets. RACER controls the probability that every selected model is wrong, but it scores each candidate in isolation, without accounting for the candidate models already in the selected set. Scoring in isolation ignores that LM failures are correlated; models trained on similar data or built from similar architectures may fail on the same queries \citep{goel2025think,chen2026cofailure}. We define two models as complementary when one succeeds on queries that the other fails to answer. A routing policy that ignores this correlation spends calls on models that cover no queries beyond those already selected, while a weaker but complementary model could have correctly answered the queries the set still misses. The second category uses the dependence between model failures or model diversity when forming sets. Classical ensemble selection and diversity measures favor models with different error patterns \citep{caruana2004ensemble,giacinto2001ensembles,kuncheva2003diversity}. Recent methods use the expected quality of the best answer or determinantal point processes to form complementary LM sets \citep{dekoninck2025unified,flexrouter2026}. These methods improve their coverage, defined as the probability that at least one selected model is correct, but they fall short of risk-controlled routing in three respects. First, they carry no distribution-free guarantee on the all-miss risk. They operate at a fixed call budget or with a fixed preference weight, and can neither accept a target risk level nor signal that one is unattainable. Second, the expected-best-answer objective of \citet{dekoninck2025unified} is evaluated from per-model estimates alone, and diversity measures and determinantal point processes score a set through pairwise model similarities, so co-failures among three or more models enter only through their pairwise projections. Third, the learned variants train query encoders and per-model embeddings in place of the router a system already runs \citep{flexrouter2026}, so changing the model pool requires retraining. Existing work therefore leaves a gap between policies that certify the all-miss risk but ignore dependence and methods that use dependence but carry no certificate. Appendix~\ref{app:related-work} has more details about both lines of work.

We improve the risk-controlled routing framework by incorporating a dependence-aware routing rule. We propose \texttt{CopRoute}, short for \emph{COoPerative ROUTEr}, a family of drop-in dependence-aware conformal routing policies that reuse the router a system already runs. Unlike earlier risk-controlled policies that score one candidate model at a time, CopRoute builds the selected set sequentially and tracks the estimated probability that every model in that set fails. CopRoute first orders the candidate models by how many design queries each one answers that the models before it missed; for each new query it then decides, from the router's predicted success probabilities, how many models along this order to call; and finally it calibrates the smallest cutoff on this order that meets the target risk. 
On ten datasets from RouterBench \citep{hu2024routerbench} and RouterEval \citep{huang2025routereval} with pools of 11 to 500 models, CopRoute attains the fewest certified calls on nine of ten datasets at a target risk of 0.05 and eight of ten at 0.10, calling a median 24 percent and up to 79 percent fewer models than RACER-G, and fewer than FlexRouter and BaRP in 45 of 48 pool comparisons.

We summarize our contributions as follows. (i) We show that the all-miss probability of a model set depends jointly on the individual model accuracies and on the dependence among model failures. Under conditional independence, scoring models in isolation is optimal at every set size, while under dependence the best set need not contain the individually best models. (ii) We introduce CopRoute and establish the three guarantees above for its model ordering, its calibrated cutoff, and its finite-sample correction. The certificate does not depend on the estimator, so any dependence estimator in the family, selected on validation before final calibration, can produce the order without weakening the guarantee. (iii) We evaluate CopRoute against RACER, FlexRouter, BaRP, and independence controls on ten datasets from RouterBench and RouterEval. In addition to the call reductions above, CopRoute lowers end-to-end inference cost by 46.6 percent relative to RACER-G at equal certified risk without lowering final-answer accuracy, and its certified call count changes least among the compared policies when the router is swapped or degraded.

\section{Preliminaries on the \(\alpha\)-VOR problem and notation}
\label{sec:setup}

Let \(\cM=\{1,\ldots,m\}\) be a pool of \(m\) language models. We write \(X\) for a random query drawn from the deployment distribution and \(x\) for one realized query. For each model \(j\in\cM\), let
$Y_j=\ind\{\text{model }j\text{ answers }X\text{ correctly}\}.$
Thus \(Y_j=1\) means that model \(j\) is correct, and \(Y_j=0\) means that it is wrong. The set
$
G(X)=\{j\in\cM:Y_j=1\}
$
contains all candidate models that answer \(X\) correctly. A routing policy, denoted by \(C\), takes a query as input and returns a set of models to call. Hence \(C(x)\subseteq\cM\), and \(|C(x)|\) is the number of inference calls made for query \(x\).

\citet{racer2026} pose the \(\alpha\)-Valid Optimal Routing (\(\alpha\)-VOR) problem as follows. For a user-specified risk target \(\alpha\in(0,1)\), the population problem is
\begin{equation}
\min_C\ \E\bigl[|C(X)|\bigr]
\quad\text{subject to}\quad
\Prob\bigl(C(X)\cap G(X)=\varnothing\bigr)\leq\alpha .
\label{eq:avor}
\end{equation}
In \eqref{eq:avor}, the subscript \(C\) under the minimum states that the optimization ranges over routing policies. This optimization does not choose a model set after the correct models are revealed. The expectation is the average number of called models for a new query \(X\). The constraint bounds the probability that the called set contains no correct model. Both quantities are taken over the deployment distribution of \(X\) and, if the routing policy is randomized, over that randomization. We call the event \(C(X)\cap G(X)=\varnothing\) an \emph{all-miss event}. The constraint is attainable only when the full pool itself meets the target, \(\Prob(G(X)=\varnothing)\leq\alpha\); when the data cannot certify a chosen target, the calibration rule below returns \texttt{NON-CERTIFIED}.

\textbf{All-miss probability and coverage for one query.}
Let \(S\subseteq\cM\) denote any proposed set of models. After fixing a query value \(x\), define
\(
q_x(S)=\Prob\bigl(Y_j=0\text{ for every }j\in S\mid X=x\bigr),
\quad
c_x(S)=1-q_x(S).
\)
Here \(q_x(S)\) is the conditional probability that every model in \(S\) is wrong, while \(c_x(S)\) is the conditional probability that at least one model in \(S\) is correct. We call \(c_x(S)\) the \emph{coverage} of \(S\) on query \(x\). 

\textbf{Ordering and prefixes.}
An ordering rule \(\pi\) takes a query \(x\) as input and returns a permutation
$
\pi_x=(\pi_{x,1},\ldots,\pi_{x,m})
$
of the model indices in \(\cM\). The model \(\pi_{x,1}\) is ranked first, \(\pi_{x,2}\) second, and so on. A fixed ordering ignores \(x\) and returns the same permutation for every query; a query-specific ordering may return a different permutation for each \(x\). For \(k\in\{0,1,\ldots,m\}\), the depth-\(k\) prefix is
\(
P_k^\pi(x)=\{\pi_{x,1},\ldots,\pi_{x,k}\},
\quad P_0^\pi(x)=\varnothing.
\)
Thus \(k\) is both the prefix depth and the number of models in that prefix; the prefix is a set, not an ordered list, and the permutation \(\pi_x\) determines which candidate models enter it.

\textbf{Sizing and entry scores.}
The ordering rule determines which candidate models are considered first, and a sizing rule determines how many of them to call for a particular query. Formally, it takes a query \(x\) and a depth \(k\), then assigns a residual \(r_k(x)\in[0,1]\) to the prefix \(P_k^\pi(x)\), with \(r_0(x)=1\). The residual estimates the probability that all \(k\) candidate models in the prefix are wrong. We require \(r_k(x)\) to be nonincreasing in \(k\), because adding another candidate model cannot increase the estimated all-miss probability. The entry score of the model at position \(k\) is
\begin{equation}
s_\pi(x,\pi_{x,k})=1-r_{k-1}(x).
\label{eq:order-size-score}
\end{equation}
The entry score in \eqref{eq:order-size-score} is the estimated coverage before the \(k\)-th candidate model is added. Since the residual is nonincreasing in depth, the entry scores are nondecreasing in depth, and because \(r_0(x)=1\) the first ordered model has entry score zero, so every certified cutoff calls at least one model. Section~\ref{sec:method} defines the specific ordering and sizing rules used by CopRoute.

\textbf{Cutoffs and nested model sets.}
For a cutoff \(\lambda\), let
$
C_\lambda(x)=\{j\in\cM:s(x,j)\leq\lambda\}.$
The set \(C_\lambda(x)\) uses the \emph{sublevel convention}; it includes every candidate model whose score is at most \(\lambda\), and the nonstrict inequality means that models with the same score enter together. These sets form a nested family. Specifically, if \(\lambda\leq\lambda'\), then
$
C_\lambda(x)\subseteq C_{\lambda'}(x)
$
for every query \(x\). The family \(\{C_\lambda(x)\}_\lambda\) is therefore a nested family of sets. Increasing the cutoff can add candidate models but cannot remove a candidate model that was already selected. For the entry scores above, each \(C_\lambda(x)\) is a prefix of \(\pi_x\), apart from an entire tied block entering at once.

\textbf{Calibration loss and empirical risk.}
Suppose the final-calibration sample contains \(n\) labeled queries \((X_i,G_i)\), where \(i\in\{1,\ldots,n\}\) and \(G_i\) is the set of models correct on \(X_i\). For a candidate cutoff \(\lambda\), define
$
L_i(\lambda)=\ind\{C_\lambda(X_i)\cap G_i=\varnothing\}.
$
Thus \(L_i(\lambda)=1\) when the selected set misses every correct model on calibration query \(i\), and it is zero otherwise. The empirical all-miss risk is
\(
\widehat R_n(\lambda)=\frac{1}{n}\sum_{i=1}^nL_i(\lambda).
\)
The hat indicates that \(\widehat R_n\) is computed from the \(n\) calibration queries rather than from the unknown deployment distribution. The calibrated cutoff is
\begin{equation}
\widehat\lambda
=\min\left\{\lambda\in\mathbb{R}:
\frac{n}{n+1}\widehat R_n(\lambda)+\frac{1}{n+1}\leq\alpha
\right\}.
\label{eq:crc}
\end{equation}
The expression inside braces is the set of cutoffs that pass the finite-sample calibration rule. The outer minimum chooses the smallest passing cutoff. Because the family is nested, this cutoff selects no more models than any other passing cutoff in the same family. If no cutoff passes, the procedure returns \texttt{NON-CERTIFIED} and deploys no routing policy under this certificate. Theorem~\ref{thm:frozen-feasible-crc} states the distribution-free guarantee for this construction, and Corollary~\ref{cor:minimal-cutoff} carries it to the calibrated score cutoff. \citet{racer2026} additionally use a virtual null model as an abstention mechanism (Appendix~\ref{app:protocol}), and using the per-model router score \(1-\widehat p_j(x)\) in the construction above reduces it to the inverse-probability (RACER-P) family of \citet{racer2026} over the real model pool.

\section{Dependence-aware routing}
\label{sec:method}


CopRoute makes an \textit{ordering} decision, which fixes the order in which models are considered, and a \textit{stopping} decision, which fixes for a given query how far down that order to call. 
While many orders and stopping rules are possible in theory, we present one default policy and one ablation of its stopping rule, and the last paragraph of this section explains this choice. 

\textbf{Greedy tally ordering.}
Let \(D\in\{0,1\}^{n_d\times m}\) be the design correctness matrix, where \(D_{ij}=1\) when model \(j\) answers design query \(i\) correctly, and let \(\bar p_j^D=n_d^{-1}\sum_iD_{ij}\). For a model set \(S\), let \(\widehat q_D(S)=n_d^{-1}\sum_i\prod_{j\in S}(1-D_{ij})\) denote its design all-miss frequency. Starting with no selected models, we append one candidate model at a time, each time choosing the model that minimizes this frequency. We denote the resulting \textit{fixed-order tally} by \(\pi^T=(\pi_1^T,\ldots,\pi_m^T)\), where \(T\) stands for tally, and write \(P_k^{\pi^T}=\{\pi_1^T,\ldots,\pi_k^T\}\). Thus
\begin{equation}
\pi_k^T\in\arg\min_{j\notin P_{k-1}^{\pi^T}}
\widehat q_D\!\left(P_{k-1}^{\pi^T}\cup\{j\}\right).
\label{eq:fixed-tally-order}
\end{equation}
Since both stopping rules use the fixed-tally order, we drop the superscript \(T\) below and write \(\pi_k\) and \(P_k\) for \(\pi_k^T\) and \(P_k^{\pi^T}\), respectively, whenever it is obvious.

\textbf{Marginal probabilities and joint failure.}
Let \(p_j(x)=\Prob(Y_j=1\mid X=x)\), and let \(\widehat p_j(x)\) be its router estimate. The true all-miss probability of prefix \(P_k\) depends on both the individual miss probabilities and their dependence. Using the correctness subcopula, \(q_x(P_k)=C^0_{P_k,x}(1-p_j(x):j\in P_k)\). If the model probabilities given a query are conditionally independent, this identity reduces to a product of individual miss probabilities, motivating the first stopping rule below. The second rule adjusts this product using observed co-failure on the design set.

\textbf{Two query-specific stopping rules}
Both rules below, $r_k^{\mathrm{ind}}(x)$ and $r_k^{\mathrm{dep}}(x)$, use the fixed order and the router estimates for the current query. The independence-based residual $r_k^{\mathrm{ind}}(x)$ multiplies the estimated miss probabilities. The dependence-based residual $r_k^{\mathrm{dep}}(x)$ additionally uses the design-set dependence ratio \(d_k=\widehat q_D(P_k)\big/\prod_{\ell=1}^{k}(1-\bar p_{\pi_\ell}^{D})\), which compares the observed prefix all-miss frequency with its independence prediction. This ratio is noisy when few design queries miss the whole prefix, so the dependence residual uses a shrunk version \(\widetilde d_k\) controlled by a support threshold \(h\). For \(h>0\) the shrunk ratio keeps the measured value when at least \(h\) design queries miss the prefix and otherwise moves it toward the independence value \(1\), whereas \(h=0\) is defined to apply no correction, so that \(\widetilde d_k=1\) at every depth. This gives
\begin{equation}
r_k^{\mathrm{ind}}(x)=\prod_{\ell=1}^{k}\left(1-\widehat p_{\pi_\ell}(x)\right),\qquad
r_k^{\mathrm{dep}}(x)=\min_{0\leq t\leq k}\operatorname{clip}_{[0,1]}\left(r_t^{\mathrm{ind}}(x)\,\widetilde d_t\right).
\label{eq:ind-dep-stopping}
\end{equation}
A dependence ratio below one indicates complementary failures; a value above one indicates more co-failure than independence predicts. 
The minimum over preceding depths prevents the estimated all-miss probability from increasing when another model is added. We choose \(h\) from the grid \(\{0,5,20,50\}\) on a separate validation split before final calibration, and because the grid includes \(h=0\), the selected value can turn the correction off altogether. For either rule, \(s(x,\pi_k)=1-r_{k-1}(x)\), and final calibration selects one global cutoff using \eqref{eq:crc}. Appendix~\ref{app:implementation} gives the complete construction and validation rule.

\textbf{The CopRoute family.}
The default policy, \textbf{CopRoute-Ind}, uses the independence-based rule (\(h=0\)) and has no hyperparameter. \textbf{CopRoute-Dep} uses the dependence-based rule with \(h\) chosen on validation; because the grid contains \(h=0\), it can reproduce CopRoute-Ind exactly, and on most redraws it does. We report CopRoute-Dep as an ablation of the stopping rule and not as a second main policy. The dependence ratio is estimated from the design queries that miss an entire prefix, which are few at deep prefixes and in large pools, and the resulting estimation noise can make the corrected rule call more models than the uncorrected one (Section~\ref{sec:experiments}).
\label{sec:method-end}

Beyond the rules described above, CopRoute supports further fixed and query-specific dependence models. Its fixed orders can be learned from empirical co-failure counts, failure correlations, rank correlations, latent-variable models, class mixtures, shrinkage estimators, or model-family information. Its query-specific variants combine the router estimates for the current query with either design-level dependence statistics or correctness labels from nearby router-training queries. These orderings are paired with independence-based, dependence-based, or local empirical stopping rules. We use the tally order with the two stopping rules above for the policies reported in the main text, because this order requires only design-set counts and has the fixed-order guarantee developed in Section~\ref{sec:ordering-optimality}. Appendix~\ref{app:implementation} describes the remaining variants.

\begin{algorithm}[t]
\caption{CopRoute}
\label{alg:coproute}
\begin{algorithmic}[1]
\REQUIRE pool \(\cM\); disjoint splits \(\mathcal D_{\mathrm{des}},\mathcal D_{\mathrm{val}},\mathcal D_{\mathrm{cal}}=\{(X_i,G_i)\}_{i=1}^n\); router \(\widehat p_j(\cdot)\); target risk \(\alpha\)
\STATE fit the tally order \(\pi^T\) on \(\mathcal D_{\mathrm{des}}\) by \eqref{eq:fixed-tally-order}
\STATE on \(\mathcal D_{\mathrm{val}}\), select the stopping rule of \eqref{eq:ind-dep-stopping}, \(r^{\mathrm{ind}}\) for the default CopRoute-Ind or \(r^{\mathrm{dep}}\) with support threshold \(h\) for the CopRoute-Dep ablation, and freeze all choices
\STATE score the model at position \(k\) of \(\pi^T\) by \(s(x,\pi^T_k)=1-r_{k-1}(x)\) for every query \(x\)
\STATE on \(\mathcal D_{\mathrm{cal}}\), set \(\widehat\lambda=\min\bigl\{\lambda:\tfrac{n}{n+1}\widehat R_n(\lambda)+\tfrac{1}{n+1}\leq\alpha\bigr\}\) as in \eqref{eq:crc}; \textbf{if} no \(\lambda\) passes, \textbf{return} \texttt{NON-CERTIFIED}
\ENSURE routing policy \(C(x)=\{j\in\cM:s(x,j)\leq\widehat\lambda\}\)
\end{algorithmic}
\end{algorithm}

\section{Theory}
\label{sec:theory}
In the previous sections, we discussed the importance of considering the co-occurrence of model failures. In this section, we establish two guarantees. First, the fixed tally order comes within a factor of four of the best possible fixed order on the design set, and no polynomial-time method can guarantee a smaller factor. Second, once the order and the stopping score are frozen, split calibration keeps the all-miss risk below the target level, and the smallest certified cutoff calls the fewest models among all certified cutoffs. Formal proofs are presented in Appendices~\ref{app:corrected-proofs} and~\ref{app:extended-theory}.

\subsection{Why fixed tally is a near-optimal ordering}
\label{sec:ordering-optimality}

For an ordering \(\pi\), let \(T_i(\pi)\) be the position of the first model that answers design query \(i\) correctly, with \(T_i(\pi)=m\) when the full pool misses that query. Summing the design all-miss estimate over all prefixes of an order gives the average position at which the first correct model appears,
\(
J_D(\pi)
=\sum_{k=0}^{m-1}\widehat q_D\bigl(P_k^\pi\bigr)
=\frac{1}{n_d}\sum_{i=1}^{n_d}T_i(\pi),
\)
and a smaller value means that a correct model appears earlier on average.

\begin{theorem}[Fixed tally matches the best achievable approximation factor]
\label{thm:mssc-tight}
Let \(D\in\{0,1\}^{n_d\times m}\) be the realized design correctness matrix. Run the fixed-tally rule of \eqref{eq:fixed-tally-order} through all \(m\) positions, using the tie rule in Section~\ref{sec:method}, and denote the resulting order by \(\pi^T\). At each step, the fixed-tally rule selects a candidate model that answers the largest number of design queries missed by the current prefix. It is therefore the greedy minimum-sum set-cover ordering, and
\begin{equation}
J_D(\pi^T)\leq 4\min_\pi J_D(\pi).
\label{eq:factor-four}
\end{equation}
Moreover, for every \(\varepsilon>0\), no polynomial-time algorithm that receives an arbitrary correctness matrix and returns a fixed model ordering can guarantee a factor \(4-\varepsilon\) for this objective unless \(\mathrm{P}=\mathrm{NP}\) \citep{feige2004mssc,bansal2023minsum}.
\end{theorem}

Theorem~\ref{thm:mssc-tight} concerns one realized design matrix and the complete fixed ordering. The factor-four guarantee does not transfer to population or test-time calls (Appendix~\ref{app:ordering-transfer} quantifies this), and a construction stopped after \(k_0<m\) positions receives the guarantee only for that truncated instance. The tally order also inherits the greedy coverage guarantee that at every depth \(k\), its design coverage is at least \(1-(1-1/k)^k\geq1-1/e\) times the best coverage attainable by any set of \(k\) models, and Appendix~\ref{app:ordering-proof} proves both results.

\subsection{Why the final routing sets are valid and efficient}

For real-model sets \(S,G\subseteq\cM\), define the standard all-miss loss
\(
L^{\mathrm{miss}}(S,G)=\ind\{S\cap G=\varnothing\}.
\)
An empty called set counts as a miss, and the loss has no abstain option, so declining to answer, or passing the query to a human reviewer, is not available as a zero-loss outcome.

\begin{theorem}[Frozen-family split certification]
\label{thm:frozen-feasible-crc}
Before inspecting any final-calibration label, let design and validation data select a finite nested family
\(
C_0(x)\subseteq C_1(x)\subseteq\cdots\subseteq C_K(x)\subseteq\cM.
\)
Condition on those data and every auxiliary random choice used to construct the family. Assume that the same measurable map is applied to every calibration query and one future query. Also assume that these \(n+1\) observations are exchangeable. Define
\(
L_i(k)=L^{\mathrm{miss}}\bigl(C_k(X_i),G_i\bigr),
\qquad
\widehat R_n(k)=\frac{1}{n}\sum_{i=1}^nL_i(k),
\)
and
\(
\mathcal K_{\mathrm{cert}}
=\left\{k:\frac{n}{n+1}\widehat R_n(k)+\frac{1}{n+1}\leq\alpha\right\}.
\)
If this set is empty, return \texttt{NON-CERTIFIED}; otherwise return \(\widehat k=\min\mathcal K_{\mathrm{cert}}\). For \(0\leq\alpha<1\),
\begin{equation}
\Prob\!\left(
\mathcal K_{\mathrm{cert}}\neq\varnothing
\text{ and }
C_{\widehat k}(X_{n+1})\cap G_{n+1}=\varnothing
\right)\leq\alpha.
\label{eq:certified-standard-miss}
\end{equation}
If certification occurs almost surely, then \(\E[L_{n+1}(\widehat k)]\leq\alpha\).
\end{theorem}

The probability in \eqref{eq:certified-standard-miss} covers the joint event that a certificate is issued and the certified set still misses. It does not bound the miss rate among only those redraws on which a certificate is issued, because conditioning on issuance would select the redraws whose calibration sample was favorable, and no distribution-free bound of this form exists for that conditional quantity. The estimator, hyperparameters, family, tie rules, seeds, and numerical draws must all be fixed before final-calibration labels are inspected. The validation split must therefore select a family before final calibration, and a procedure that calibrates every candidate and then selects among the calibrated policies is not covered. 
\citet{racer2026} augment the pool with a null action whose loss is zero by definition, so their construction controls a different quantity from the all-miss risk here, under which declining to call any model counts as a miss.

\begin{corollary}[Minimum certified cutoff]
\label{cor:minimal-cutoff}
Let a fixed score \(s\) generate the frozen sublevel family \(C_\lambda(x)=\{j:s(x,j)\leq\lambda\}\), and let
\(
\Lambda_{\mathrm{cert}}
=\left\{\lambda:
\frac{n}{n+1}\widehat R_n(\lambda)+\frac{1}{n+1}\leq\alpha
\right\}.
\)
If \(\Lambda_{\mathrm{cert}}\neq\varnothing\), the smallest certified cutoff \(\widehat\lambda=\min\Lambda_{\mathrm{cert}}\) is attained and satisfies
\(
C_{\widehat\lambda}(x)\subseteq C_\lambda(x)
\qquad
\text{for every }\lambda\in\Lambda_{\mathrm{cert}}
\text{ and every }x.
\)
Therefore, among all certified cutoffs in this frozen family, \(\widehat\lambda\) calls the fewest models. The same holds for any cost that never decreases when a model is added to the called set.
\end{corollary}

Essentially this means that once the order and the stopping score are frozen, every certified cutoff produces nested sets, the smallest certified cutoff produces the smallest of them for every query, and calibration therefore never trades calls against risk within the family but simply finds its cheapest certified member. The corollary is exact within one frozen family; it is not an optimality claim across different orderings, stopping rules, or policy classes. Appendix~\ref{app:calibration-proof} proves Theorem~\ref{thm:frozen-feasible-crc} and Corollary~\ref{cor:minimal-cutoff}, and also shows that the correction \(1/(n+1)\) cannot be reduced by any deterministic calibration rule of this form, so no risk target below \(1/(n+1)\) is certifiable.

\section{Experiments}
\label{sec:experiments}

We answer four questions, namely whether CopRoute calls fewer models than existing routing policies at a matched certified risk, whether dependence-corrected stopping improves on the independence-based default, whether an ordering fit on one dataset remains useful on another, and how much the results depend on the router that provides the scores.

\textbf{Setup.}
We evaluate on RouterBench \citep{hu2024routerbench}, whose pool of eleven models answers ARC, HellaSwag, MMLU, and WinoGrande, and we do not use its two remaining tasks, since the GSM8K release contains a single distinct prompt so that no prompt-conditioned router can be fitted, and on MBPP even the full pool misses 13.4 percent of queries so that neither target is certifiable. We also evaluate on RouterEval \citep{huang2025routereval}, which records per-query correctness for 3{,}811 models. From RouterEval we use six benchmarks with uniformly subsampled pools of 100, 200, 300, and 500 models; the main table reports the 100-model pool. Throughout, CopRoute without a suffix denotes the default policy CopRoute-Ind, and the four baselines are RACER-G, RACER-P, FlexRouter, and BaRP. A seventh RouterEval benchmark, MATH-intermediate-algebra, certifies for no policy at the targets we study and appears in Appendix~\ref{app:calibrated-results}. All policies use the same KNN router scores, except that FlexRouter \citep{flexrouter2026} and BaRP \citep{wang2025barp} train their own routers, which we retrain with five seeds. Every cell uses 20 redraws with disjoint design, router, validation, final-calibration, and test partitions; all fitting and tuning happen before final calibration, after which the policy is frozen. We report the mean number of called models over certified redraws at each target risk $\alpha$, and we state the certification rate wherever any redraw fails to certify.

\begin{table}[t]
\centering
\caption{Certified mean calls at matched target risk, mean (dispersion) over 20 redraws; the dispersion is the standard error, or the standard deviation across five training seeds for FlexRouter and BaRP. CopRoute is the default policy CopRoute-Ind. The \colorbox{green!18}{best} value in each row is shaded, with exact ties shaded together, and the next distinct value is \underline{underlined}. The last column is the exact-independence control, which orders models by router score alone and is reported but not ranked. RouterBench datasets (first four rows) use the 11-model pool and RouterEval datasets the 100-model pool.}
\label{tab:main-calls}
\scriptsize
\begin{subtable}[t]{0.495\textwidth}
\centering
\caption{Target risk $\alpha=0.05$.}
\label{tab:main-calls-005}
\setlength{\tabcolsep}{2.5pt}
\resizebox{\linewidth}{!}{%
\begin{tabular}{@{}l r rr rr | r@{}}
\toprule
 & & \multicolumn{2}{c}{RACER} & & & control \\
\cmidrule(lr){3-4}\cmidrule(lr){7-7}
dataset & CopRoute & G & P & Flex. & BaRP & Indep. \\
\midrule
ARC & \cellcolor{green!18}1.25 (0.13) & 1.39 (0.17) & 2.44 (0.14) & \underline{1.33 (0.01)} & 1.37 (0.02) & 1.48 (0.25) \\
HellaSwag & \cellcolor{green!18}3.19 (0.09) & \underline{3.87 (0.08)} & 4.25 (0.10) & 4.65 (0.13) & 5.14 (0.40) & 3.95 (0.11) \\
MMLU & \cellcolor{green!18}7.75 (0.27) & 8.66 (0.10) & \underline{8.64 (0.04)} & 9.67 (0.04) & 9.74 (0.08) & 9.45 (0.18) \\
WinoGrande & \cellcolor{green!18}2.40 (0.03) & 5.31 (0.20) & 5.10 (0.21) & \underline{4.56 (0.15)} & 4.62 (0.18) & 4.81 (0.19) \\
BBH-boolean & \cellcolor{green!18}2.63 (0.29) & 4.75 (0.58) & 5.20 (0.46) & \underline{2.70 (0.22)} & 3.01 (0.31) & 4.15 (0.57) \\
BBH-navigate & \cellcolor{green!18}6.42 (3.63) & 31.3 (5.30) & 31.8 (5.15) & 24.3 (1.68) & \underline{22.0 (2.97)} & 17.3 (1.63) \\
GPQA & \cellcolor{green!18}17.2 (2.14) & \underline{21.0 (1.85)} & \underline{21.0 (1.85)} & 26.3 (0.67) & 25.7 (5.28) & 22.6 (2.19) \\
IFEval & \cellcolor{green!18}22.1 (5.97) & \underline{24.7 (5.04)} & 26.3 (5.41) & 30.3 (3.84) & 26.8 (3.29) & 24.3 (5.19) \\
MATH-algebra & 35.1 (14.7) & \underline{34.7 (18.8)} & \cellcolor{green!18}32.0 (15.5) & 45.2 (2.98) & 38.1 (13.1) & 28.1 (12.3) \\
MuSR & \cellcolor{green!18}57.2 (8.76) & \underline{75.0 (4.28)} & 77.0 (4.81) & 86.3 (3.15) & 87.8 (4.01) & 69.3 (7.55) \\
\bottomrule
\end{tabular}}
\end{subtable}
\hfill
\begin{subtable}[t]{0.495\textwidth}
\centering
\caption{Target risk $\alpha=0.10$.}
\label{tab:main-calls-01}
\setlength{\tabcolsep}{2.5pt}
\resizebox{\linewidth}{!}{%
\begin{tabular}{@{}l r rr rr | r@{}}
\toprule
 & & \multicolumn{2}{c}{RACER} & & & control \\
\cmidrule(lr){3-4}\cmidrule(lr){7-7}
dataset & CopRoute & G & P & Flex. & BaRP & Indep. \\
\midrule
ARC & \cellcolor{green!18}1.00 (0.00) & \underline{1.06 (0.02)} & 1.71 (0.08) & \cellcolor{green!18}1.00 (0.00) & \cellcolor{green!18}1.00 (0.00) & 1.00 (0.00) \\
HellaSwag & \underline{1.80 (0.03)} & \cellcolor{green!18}1.77 (0.03) & 2.24 (0.03) & 2.51 (0.08) & 2.61 (0.24) & 1.80 (0.03) \\
MMLU & \cellcolor{green!18}2.99 (0.05) & 5.64 (0.07) & 6.45 (0.04) & \underline{3.89 (0.06)} & 4.06 (0.10) & 3.87 (0.07) \\
WinoGrande & \cellcolor{green!18}1.80 (0.04) & 3.10 (0.15) & 3.00 (0.13) & \underline{2.38 (0.02)} & 2.61 (0.20) & 2.59 (0.13) \\
BBH-boolean & \cellcolor{green!18}1.13 (0.05) & 2.25 (0.18) & 3.24 (0.35) & \underline{1.23 (0.03)} & 1.34 (0.11) & 1.57 (0.18) \\
BBH-navigate & \cellcolor{green!18}1.86 (0.07) & 7.36 (1.20) & 7.61 (1.25) & \underline{6.35 (1.24)} & 6.90 (1.93) & 5.95 (1.00) \\
GPQA & \cellcolor{green!18}9.65 (0.95) & 13.3 (0.93) & 13.3 (0.93) & 13.2 (0.15) & \underline{13.1 (2.46)} & 14.3 (0.99) \\
IFEval & \cellcolor{green!18}2.86 (0.32) & \underline{3.76 (0.48)} & 5.66 (0.53) & 7.31 (0.89) & 5.16 (1.37) & 3.22 (0.44) \\
MATH-algebra & 11.4 (3.74) & \underline{11.0 (3.29)} & \cellcolor{green!18}10.4 (3.02) & 12.1 (1.51) & 17.8 (7.46) & 9.79 (2.81) \\
MuSR & \cellcolor{green!18}12.7 (1.56) & \underline{44.6 (2.81)} & 53.3 (1.67) & 49.7 (5.11) & 55.2 (3.95) & 32.1 (3.16) \\
\bottomrule
\end{tabular}}
\end{subtable}
\end{table}

\textbf{Certified calls at matched risk.}
Table~\ref{tab:main-calls} gives the main comparison. At $\alpha=0.05$ every policy certifies all 20 redraws except on IFEval (19), MATH-algebra (3), and MuSR (8), identically across policies, whereas at $\alpha=0.10$ every cell certifies all 20. CopRoute has the lowest calls on nine of the ten datasets at $\alpha=0.05$ and on eight at $\alpha=0.10$. The improvements are largest on the reasoning benchmarks. At $\alpha=0.10$, MuSR needs 12.7 calls against 44.6 for RACER-G, the closest baseline, and 55.2 for BaRP, and BBH-navigate needs 1.9 against 6.4 for FlexRouter and 7.6 for RACER-P, while MMLU and WinoGrande need roughly half the calls of either RACER variant. The exceptions are MATH-algebra, where RACER-P is better by 3.1 calls at $\alpha=0.05$ and by 1.0 call at $\alpha=0.10$, and HellaSwag at $\alpha=0.10$, where RACER-G is better by 0.03 calls. The exact-independence control orders models by router score alone and stops with the same independence product. It matches CopRoute or calls more models on nine datasets, by up to 19.4 calls more on MuSR, and is better on MATH-algebra by 7.0 calls at $\alpha=0.05$ and 1.6 calls at $\alpha=0.10$, so on that benchmark the design co-failure counts order the pool less well than the query-specific router scores do. Figure~\ref{fig:radar} extends the comparison to the 500-model pool, which the table does not show, and reports relative call efficiency, the per-dataset best calls divided by each policy's calls, so that the winner lies on the rim. Over the 200-, 300-, and 500-model pools (Figure~\ref{fig:radar-pools} in Appendix~\ref{app:calibrated-results}), CopRoute is on the rim on 29 of the 36 axes and FlexRouter on the remaining seven, five of them BBH-boolean or GPQA cells.

\begin{figure}[t]
\centering
\includegraphics[width=\linewidth]{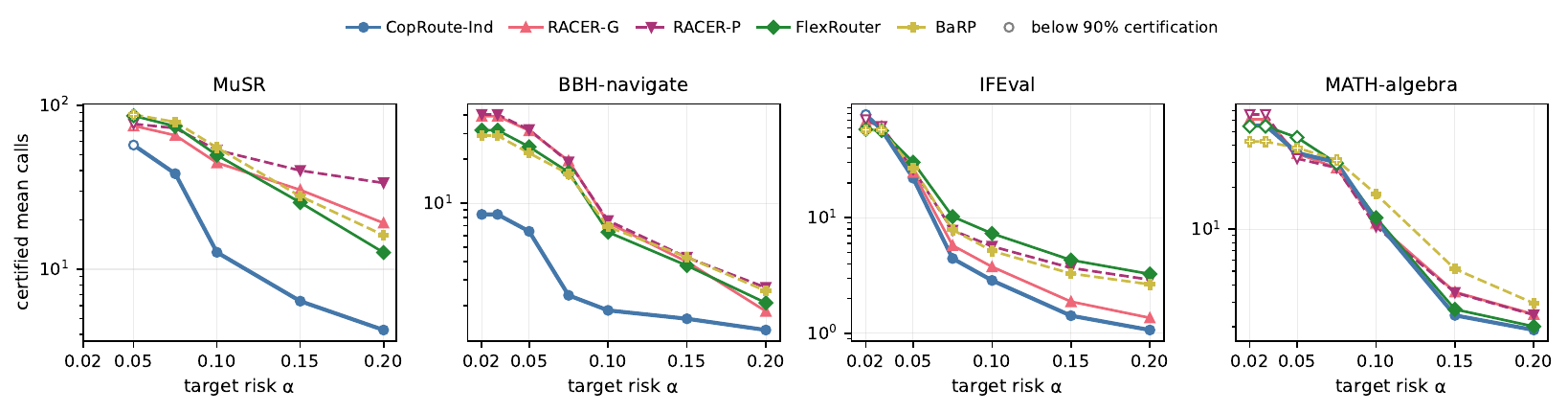}
\caption{Certified mean calls against the target risk, 100-model pool, logarithmic vertical axis; lower is better. Markers are open where fewer than 90 percent of redraws certify, and a missing point means none certifies.}
\label{fig:sweep-calls}
\end{figure}

\textbf{Risk sweep.}
Figure~\ref{fig:sweep-calls} plots certified calls over the full target grid. The CopRoute curve is the lowest across the sweep on MuSR, BBH-navigate, and IFEval, and the advantage widens as the target loosens. The sweep also shows where certification becomes infeasible, since no policy certifies MuSR below $\alpha=0.05$ at this pool and every policy certifies fewer redraws at the tightest targets. Calibration can certify a target only if the full pool misses few enough calibration queries, namely when $(c_M+1)/(n+1)\leq\alpha$ for the full-pool miss count $c_M$ among the $n$ calibration queries. This boundary is therefore set jointly by the calibration-set size and by how often even the whole pool fails, and it is shared by all policies because it depends on neither the order nor the stopping rule.

\textbf{Dependence-corrected stopping.}
Appendix~\ref{app:full-results} (Table~\ref{tab:dep-ablation}) isolates the two components of the policy on the cells of Table~\ref{tab:main-calls}. Replacing the independence-based stopping rule by the dependence-corrected rule of CopRoute-Dep changes nothing on 12 of the 20 cells, because validation selects \(h=0\) (on 10 to 19 of 20 redraws on IFEval and MuSR, the two benchmarks for which the selected values were recorded), and where it does change the result, it lowers calls on MATH-algebra at $\alpha=0.10$ by 1.7 with an interval that includes zero and raises them on IFEval at $\alpha=0.05$ by 5.6, so CopRoute-Ind is the default. Replacing the tally order by the design-marginal order, which ranks models by design accuracy alone, costs up to 21 calls on BBH-navigate and 27 on MuSR, so the gain over the baselines comes from the co-failure order rather than from the stopping correction.

\begin{figure}[t]
\begin{minipage}[t]{0.555\textwidth}
\centering
\includegraphics[width=\linewidth]{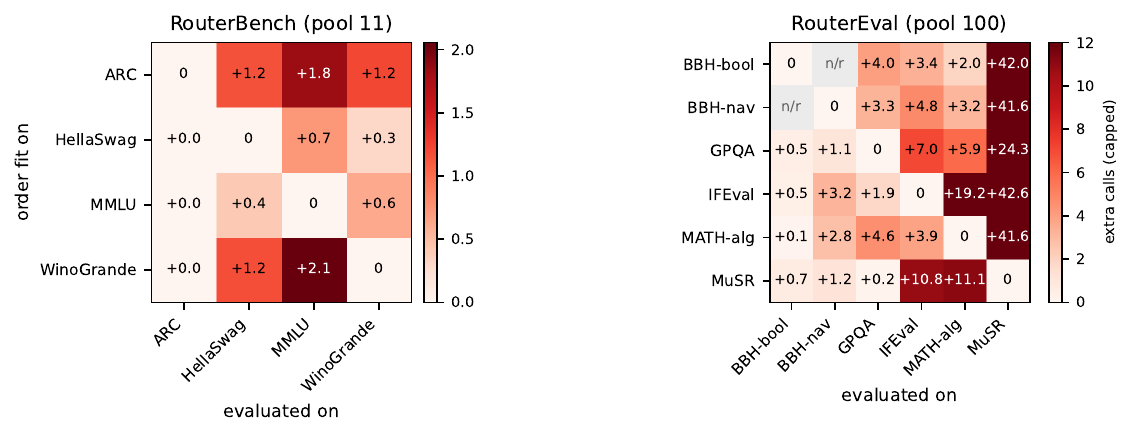}
\caption{Extra certified calls for CopRoute at $\alpha=0.10$ when the ordering is fitted on the row dataset and recalibrated on the column dataset. The color scale is capped at 12 and the annotated values are exact. Transfers between RouterBench and RouterEval are not possible because the two benchmarks have disjoint model pools, and the two pairs within the BBH family, marked n/r, were not run so that no transfer is near in-domain. Appendix~\ref{app:transfer} reports the same map for every policy.}
\label{fig:ood-premium}
\end{minipage}\hfill
\begin{minipage}[t]{0.42\textwidth}
\centering
\includegraphics[width=\linewidth]{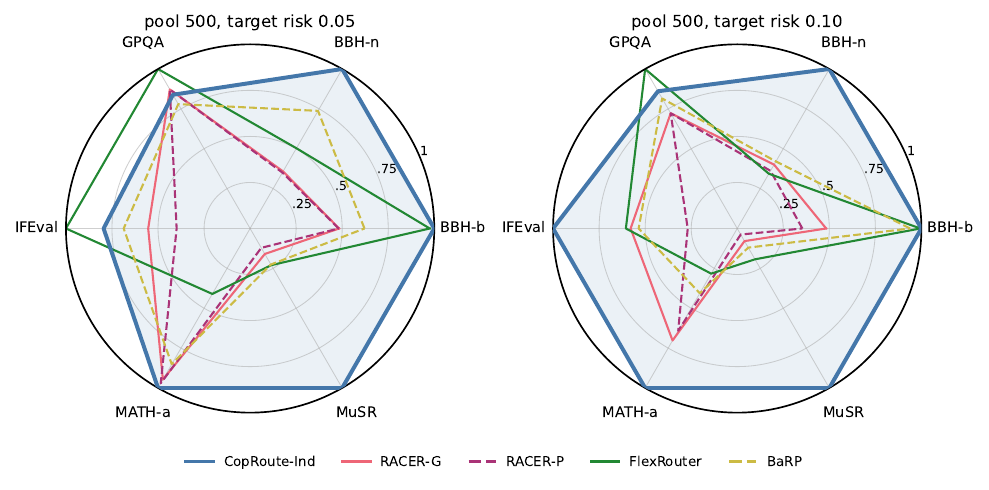}
\caption{Relative call efficiency at the 500-model pool of the six RouterEval datasets, defined as the best certified calls on each axis divided by the policy's certified calls, so that the per-axis winner lies on the rim at 1. Figure~\ref{fig:radar-pools} shows the 200- and 300-model pools.}
\label{fig:radar}
\end{minipage}
\end{figure}

\textbf{Transfer across datasets.}
Figure~\ref{fig:ood-premium} reports the extra certified calls when the ordering is fitted on one dataset and deployed on another. On RouterBench every transfer costs at most 2.1 extra calls, and most RouterEval transfers cost fewer than seven; transfers into MuSR are the most expensive at 24 to 43 extra calls. FlexRouter and BaRP degrade least under this shift because their orderings encode little dataset-specific structure (Appendix~\ref{app:transfer}); CopRoute degrades the most because its ordering is fitted to one dataset's co-failure structure, the property that makes it efficient in-domain. The ranking narrows, and averaged over all out-of-domain pairs CopRoute remains the lowest-call policy on RouterBench at both targets and on RouterEval at $\alpha=0.05$ (33.3 calls against 34.6 to 37.6), while at $\alpha=0.10$ it is 1.1 calls behind FlexRouter. Reusing a source-calibrated cutoff inflates the realized risk for every policy, while recalibrating on the deployment dataset restores the guarantee in every cell. The ordering can therefore be reused across related datasets, but the cutoff must be recalibrated on the deployment distribution.

\textbf{Changing the router.}
We repeat the main comparison with RouterDC \citep{chen2024routerdc} in place of the KNN router. Swapping the router moves CopRoute by at most 0.6 calls on any dataset, with a median absolute change of 0.2 calls at $\alpha=0.10$, while RACER-G and RACER-P move by up to 31 calls with median changes of 2 to 3 calls. On RouterDC's own scores, CopRoute calls fewer models than RACER-G on all ten datasets at $\alpha=0.10$.

\begin{table}[t]
\begin{minipage}[t]{0.485\textwidth}
\centering
\caption{Extra certified calls of each alternative order relative to the tally order under the same stopping rule, weighted by certified redraws over the ten datasets; negative (bold) favors the alternative.}
\label{tab:estimator-family}
\scriptsize
\setlength{\tabcolsep}{4pt}
\resizebox{\linewidth}{!}{%
\begin{tabular}{@{}l rr rr@{}}
\toprule
 & \multicolumn{2}{c}{Ind stopping} & \multicolumn{2}{c}{Dep stopping} \\
\cmidrule(lr){2-3}\cmidrule(lr){4-5}
order & all pools & pool 100 & all pools & pool 100 \\
\midrule
Bernoulli mixture & +0.42 & \textbf{-0.65} & +0.36 & \textbf{-0.39} \\
Student-t copula & +1.77 & \textbf{-0.68} & +1.35 & \textbf{-0.67} \\
Gaussian copula & +1.97 & \textbf{-0.09} & +1.81 & \textbf{-0.10} \\
rank correlation & +2.04 & +0.26 & +1.83 & +0.20 \\
double fault & +1.59 & +0.13 & +1.13 & +0.00 \\
error correlation & +6.25 & +0.98 & +5.50 & +1.23 \\
accuracy-weighted random & +4.28 & +2.54 & +4.04 & +2.58 \\
random control & +7.78 & +4.19 & +7.24 & +4.00 \\
\bottomrule
\end{tabular}}

\end{minipage}\hfill
\begin{minipage}[t]{0.475\textwidth}
\centering
\caption{Cost and final-answer accuracy on RouterBench MMLU at matched certified risk, means over 20 redraws; the \colorbox{green!18}{best} value per column is shaded. Standard errors and uncalibrated references appear in Table~\ref{tab:cost-case-study}.}
\label{tab:cost-main}
\scriptsize
\setlength{\tabcolsep}{2.6pt}
\resizebox{\linewidth}{!}{%
\begin{tabular}{@{}l rrr rrr@{}}
\toprule
 & \multicolumn{3}{c}{$\alpha=0.05$} & \multicolumn{3}{c}{$\alpha=0.10$} \\
\cmidrule(lr){2-4}\cmidrule(lr){5-7}
policy & calls & acc.\ (\%) & USD/1k & calls & acc.\ (\%) & USD/1k \\
\midrule
CopRoute-Ind & \cellcolor{green!18}7.75 & \cellcolor{green!18}80.44 & \cellcolor{green!18}3.41 & \cellcolor{green!18}2.99 & \cellcolor{green!18}80.97 & \cellcolor{green!18}1.90 \\
RACER-G & 8.66 & 80.19 & 4.24 & 5.64 & 80.18 & 3.55 \\
RACER-P & 8.64 & 80.19 & 4.22 & 6.45 & 80.13 & 3.66 \\
exact independence & 9.45 & 80.19 & 4.28 & 3.87 & 80.19 & 2.83 \\
design-marginal order & 9.44 & 80.20 & 4.30 & 3.76 & 80.66 & 2.86 \\
\midrule
call every model & 11.00 & 80.19 & 4.47 & 11.00 & 80.19 & 4.47 \\
best single model (1 call) & 1.00 & 81.39 & 1.43 & 1.00 & 81.39 & 1.43 \\
\bottomrule
\end{tabular}}

\end{minipage}
\end{table}

\textbf{Alternative estimators.}
CopRoute is a family of routing policies built from two interchangeable components. An estimator fitted on the design split produces the model order, and a stopping rule converts that order into a per-query candidate set before the shared calibration step certifies the cutoff. In Algorithm~\ref{alg:coproute}, a family variant changes only step 1, where an alternative estimator fits the order in place of the tally count, or step 2, where validation selects the stopping rule; scoring, calibration, and deployment are unchanged. Table~\ref{tab:estimator-family} compares the alternative orders against the tally default under both stopping rules. No alternative improves on the tally order in aggregate; the closest, a Bernoulli mixture, uses 0.36 to 0.42 more calls, and a seeded random order uses about 7 more. At the 100-model pool the Bernoulli-mixture and Student-t-copula orders use 0.4 to 0.7 fewer calls than the tally order, which indicates that estimators with more parameters become useful when the design set is large relative to the pool (definitions in Appendix~\ref{app:estimators}, complete grids in Appendix~\ref{app:full-results}).

\textbf{Cost and final accuracy.}
The certified guarantee bounds coverage, and Table~\ref{tab:cost-main} therefore also reports final-answer accuracy and cost on RouterBench MMLU, with the answer taken from the called model with the highest router score and each call priced from public per-model prices. At $\alpha=0.10$, CopRoute answers 80.97 percent of queries at USD 1.90 per thousand queries, against 80.18 percent at USD 3.55 for RACER-G, a 46.6 percent cost reduction with a 0.79-point accuracy gain at the same certified risk. Final-answer accuracy depends on the answer selector and the called set, and the two uncalibrated reference rows show that the certificate does not determine it. Calling every model returns the router's top choice at 80.19 percent, whereas the single model with the best design accuracy is both cheaper and more accurate than every calibrated policy, at USD 1.43 and 81.39 percent, but misses 18.6 percent of queries and therefore cannot certify either target. The calibrated policies add the certified coverage guarantee, and among them CopRoute reaches it at the lowest cost. Appendix~\ref{app:cost} reports the standard errors and a cost-weighted variant for pools with heterogeneous prices.

\section{Limitations and Conclusion}
\label{sec:conclusion}

\textbf{Summary.} We pair a joint routing objective with a family of policies that learn which models add coverage beyond those already selected. The correctness-subcopula identity defines the objective, the tally order has a tight approximation guarantee for it, and split calibration certifies the deployed cutoff. At matched certified risk, CopRoute is the lowest-call routing policy on nine of ten datasets at $\alpha=0.05$, on eight of ten at $\alpha=0.10$, and on 29 of the 36 larger-pool cells, and it keeps that lead across the target sweep, across routers, and at pools of up to 500 models. The gain comes from the co-failure order, since replacing it by the design-marginal order costs up to 27 calls, while swapping the router changes CopRoute by at most 0.6 calls and degrading its scores by at most 0.4 calls.

\textbf{Limitations.} The certificate bounds the average all-miss risk of the candidate set over the deployment distribution, which is the quantity a routing layer controls, and per-query guarantees and certificates for the answer returned by an aggregator are separate objects outside the scope of this framework. The fitted order and the calibrated cutoff are specific to the deployment pool and distribution, both are inexpensive to refit, and Appendices~\ref{app:transfer} and~\ref{app:cost} quantify the cost of reusing an order across related datasets or after a pool change. The dependence-corrected stopping rule is estimated from the design queries that miss an entire prefix, which are few at deep prefixes, and it does not improve on the independence-based default, so the open question is a tail estimator that uses more of the design set. Our evaluation uses recorded correctness labels, which permit the 20-redraw protocol across ten datasets and pools of up to 500 models, and it prices calls on one benchmark, while latency and response-dependent aggregation are left to deployment studies. Appendix~\ref{app:limitations} discusses each point in more detail.


\subsection*{AI use statement}
In this work, we used generative AI coding assistants to implement parts of the experiment code and the scripts that regenerate tables and figures from recorded result files; the authors reviewed and tested this code, and every reported number was regenerated from the recorded results. We have not used generative AI tools to develop the theoretical framework, to formulate or prove mathematical claims, to propose hypotheses, to design the experimental methodology, to generate data, or to interpret results. We also used generative AI tools to edit the paper for readability and to format references, and the authors reviewed every edited passage. We take responsibility for the final content of this work, including text, claims, and artifacts produced with the aid of generative AI.

\subsection*{Ethics statement}
This work uses only publicly released benchmark artifacts, RouterBench and RouterEval, which record per-query correctness of published language models; no human subjects, personal data, or new model outputs are involved. The intended use is reducing the number of model calls needed to reach a user-specified reliability target, which lowers the cost of serving language models. The certificates studied here bound the average all-miss risk over a deployment distribution; they are not per-query guarantees and are not a substitute for domain-specific safety review in high-stakes applications.

\subsection*{Reproducibility statement}
All experiments run on publicly available data. Section~\ref{sec:experiments} and Appendix~\ref{app:protocol} specify the split construction, the 20-redraw protocol, the calibration rule, and the certification criterion; Appendix~\ref{app:estimators} defines every estimator and its hyperparameter grid; Appendix~\ref{app:repro} lists seeds, split fractions, and the recorded result file behind each reported table. Every table and figure in the paper is generated by a script from a recorded result file. The code, the recorded result files, and one regeneration script per reported artifact are available at \url{https://github.com/kasakh/coproute}.

\bibliographystyle{preprint}
\bibliography{refs}

\appendix
\clearpage
\begingroup
\hypersetup{hidelinks}
\setcounter{tocdepth}{1}
{\Large\bfseries Appendix contents}\par\vspace{0.6em}
\titlecontents{section}[1.6em]{\smallskip}{\contentslabel{1.6em}}{\hspace*{-1.6em}}{\titlerule*[0.6pc]{.}\contentspage}
\startcontents[appendix]
\printcontents[appendix]{}{1}{}
\endgroup
\clearpage
\section{Extended related work}
\label{app:related-work}

\paragraph{Language-model routing and cascades.}
Most routing work selects a single model per query from candidates that differ in quality and inference cost. FrugalGPT combines a routing stage with a cascade to reduce inference cost \citep{chen2023frugalgpt}, RouteLLM learns a preference-based decision between models \citep{ong2024routellm}, HybridLLM uses a quality threshold to route between a smaller and a larger model \citep{ding2024hybridllm}, and BEST-Route allocates a query-specific test-time compute budget \citep{ding2025bestroute}. Learned routers include benchmark-trained predictors \citep{shnitzer2023routing}, reward-guided expert routing \citep{lu2024expert}, the dual-contrastive query and model representations of RouterDC \citep{chen2024routerdc}, the graph representation of GraphRouter \citep{feng2025graphrouter}, and the compact model representations of EmbedLLM \citep{zhuang2025embedllm}. Language-model cascades instead follow a sequence and stop when an uncertainty or quality rule accepts an answer \citep{dohan2022cascades,gupta2024cascades}. Surveys and benchmark studies organize this literature \citep{srivatsa2024routing,varangotreille2025survey,hu2024routerbench,huang2025routereval}. Most of this literature either returns a single model or follows a prescribed sequence, and it does not directly optimize the event that an entire selected set fails.

Bandit routers study the setting in which the system observes the outcome only of the candidate model it called, so that feedback on the models not called is never available. Contextual-bandit algorithms learn from this partial feedback \citep{li2010linucb,zhou2020neuralucb,riquelme2018deepbandits}. C2MAB-V applies an online bandit algorithm to cost-aware multi-LLM selection \citep{dai2024c2mabv}. LLM Bandit conditions a dynamic router on the user's quality and cost preference \citep{li2025llmbandit}. BaRP combines a preference-conditioned router with a policy-gradient algorithm that observes only the chosen model's feedback \citep{wang2025barp}. BaRP is designed as a single-model policy, so our learned-baseline experiments use its frozen quality logits only to define a model ordering for the common set-routing evaluation, whereas its original single-model quality and cost objective is reported separately.

\paragraph{Risk-controlled routing and conformal selection.}
RACER uses a calibration set to convert any base router into a set-valued routing policy \citep{racer2026}; it chooses a threshold for the score assigned to every candidate language model, and the threshold produces nested model sets and controls the probability that every selected model is wrong. CP-Router uses conformal uncertainty to choose between language and reasoning models \citep{su2025cprouter}, and because it selects between exactly two systems without forming candidate sets over a pool, it has no direct counterpart in the set-routing comparison. CR2 studies a cost-aware conformal loss for a wireless device-edge router \citep{xue2026cr2}. These methods establish the value of a post-hoc calibration step, but their candidate families are constructed from scores attached to individual models. CopRoute retains the risk target and final calibration step of RACER and changes how the candidate family is ordered and sized.

This calibration step draws on the split-conformal framework, which gives distribution-free uncertainty guarantees \citep{vovk2005algorithmic,angelopoulos2021gentle}. Conformal risk control extends these ideas to bounded monotone losses \citep{angelopoulos2022crc}, and Learn then Test provides a multiple-testing procedure for controlling risk \citep{angelopoulos2021ltt}. Conformal language modeling constructs sets of acceptable generated responses \citep{quach2023conformal}. Other language-model applications use conformal methods for multiple-choice questions and efficient LM decisions \citep{kumar2023conformalmcqa,vishwakarma2025prune}. Adaptive and regularized adaptive prediction sets use cumulative class probabilities to vary set size across inputs \citep{romano2020aps,angelopoulos2021raps}. Their labels are mutually exclusive, whereas several routed models may be correct on the same query. Selecting among conformal constructions after the calibration step can invalidate the original guarantee; and valid selection among conformal sets has been studied as a problem in its own right \citep{hegazy2025validselection}. CopRoute chooses its estimator and other settings before the final calibration step, then calibrates only the frozen nested family.

\paragraph{Complementary model sets and aggregation.}
\citet{dekoninck2025unified} study the tradeoff between the quality of an answer and its inference cost, and their set objective includes the expected quality of the best answer in a selected set, which for binary outcomes equals the probability that at least one selected model is correct. Their estimate of this joint quantity uses per-model quality estimates and per-model uncertainty variances fitted on validation data; the co-occurrence of failures across models is not modeled. FlexRouter learns complementary, variable-size model sets with a determinantal point process \citep{flexrouter2026}; it combines a query-dependent competence score with a learned model-similarity score, then uses greedy determinant gains to select a set. Its original objective raises the probability of selecting at least one correct model and reduces redundancy, but does not provide a distribution-free all-miss certificate. Our FlexRouter reconstruction preserves the DPP ordering mechanism and places the resulting frozen ordering inside the same set-routing evaluation used for CopRoute.

Complementary set routing also connects to the ensemble-selection literature. Greedy selection from a model library can improve aggregate predictions \citep{caruana2004ensemble}, while double-fault and error-correlation measures quantify whether classifiers fail together \citep{giacinto2001ensembles,kuncheva2003diversity}. \citet{turkmen2026information} select models with a mutual-information objective, \citet{zhang2026complementary} train complementary proposers for a summarizer, and \citet{chen2026cofailure} measures how co-failure limits routers, voting rules, and mixtures of agents. LLM-Blender ranks and fuses responses after several models have already been called \citep{jiang2023llmblender}. Self-consistency aggregates multiple reasoning paths from one model \citep{wang2023selfconsistency}, while Mixture-of-Agents uses outputs from several agents in a layered architecture \citep{wang2024moa}. These aggregation methods address how to combine available responses, whereas CopRoute determines which candidate models to call and controls the all-miss risk.

\paragraph{Dependence models, diversity, and ordering.}
Copulas separate marginal behavior from dependence \citep{sklar1959,nelsen2006}. For binary outcomes, latent Gaussian models trace to the tetrachoric-correlation construction \citep{pearson1900}. The correctness-subcopula identity of Section~\ref{sec:method} is this separation applied to binary correctness, in which the marginals are the per-model accuracies that the router estimates and the subcopula is the co-failure structure that the design set estimates. The empirical tally estimator works with observed correctness patterns and never fits a parametric copula. Correlation-aware contextual bandits also model statistical relationships among model rewards, but optimize an online surrogate-reward objective rather than calibrated set failure \citep{sridhar2026correlation}.

Determinantal point processes are a probabilistic model for diverse subsets \citep{kulesza2012dpp}. Prior work studies learning DPP kernels \citep{affandi2014dppkernels} and greedy MAP inference \citep{gillenwater2012mapdpp,han2017greedymap}. FlexRouter brings this quality-diversity construction to language-model routers. CopRoute orders models with the tally score, which counts how many previously unanswered design queries a candidate model answers.

The coverage of a selected model set is a monotone submodular function, so greedy fixed-cardinality selection inherits the classical approximation guarantee \citep{nemhauser1978}. Reaching a required coverage level connects to submodular set cover \citep{wolsey1982submodularcover}. Optimizing one fixed ordering for the average position of the first correct model is an instance of minimum-sum set cover \citep{feige2004mssc,bansal2023minsum}, which gives the tight polynomial-time approximation guarantee used by CopRoute. These deterministic ordering results do not themselves provide a statistical certificate; that guarantee comes from the independent conformal calibration step.

\clearpage
\section{Notation reference}
\label{app:notation}

Table~\ref{tab:notation} collects the symbols used throughout the main text and the appendix. Symbols that appear only inside individual theorems or proofs, and the per-method parameters of the learned baselines, are defined where they first occur.

\begingroup
\footnotesize
\setlength{\tabcolsep}{4pt}
\renewcommand{\arraystretch}{1.12}
\begin{longtable}{@{}>{\raggedright\arraybackslash}p{0.21\linewidth}>{\raggedright\arraybackslash}p{0.50\linewidth}>{\raggedright\arraybackslash}p{0.22\linewidth}@{}}
\caption{Notation used in the paper.}
\label{tab:notation}\\
\toprule
Symbol & Meaning & Note \\
\midrule
\endfirsthead
\multicolumn{3}{c}{\tablename~\thetable\ continued}\\
\toprule
Symbol & Meaning & Note \\
\midrule
\endhead
\midrule
\multicolumn{3}{r}{Continued on the next page}\\
\endfoot
\bottomrule
\endlastfoot

\multicolumn{3}{@{}l}{\textbf{Basic routing objects}}\\
\(\cM=\{1,\ldots,m\}\) & Pool of candidate language models. & \(m\) is the number of models. \\
\(i,j,k,\ell\) & Indices for a labeled query, a model, a prefix depth, and a position inside a prefix, respectively. & The range follows the equation in which an index appears. \\
\(X,x\) & A random query and one realized query, respectively. & Uppercase denotes a random variable. \\
\(Y_j\), \(Y=(Y_1,\ldots,Y_m)\) & Binary indicator that model \(j\) answers \(X\) correctly, and the vector of all model outcomes. & \(Y_j\in\{0,1\}\). \\
\(G(X),G_i\) & Models that answer query \(X\), or labeled query \(i\), correctly. & Both are subsets of \(\cM\). \\
\(S,T\subseteq\cM\) & Generic sets of models. & Local proofs may use either letter. \\
\(C(X)\) & Set of models returned by a routing policy for query \(X\). & The objective minimizes \(\E|C(X)|\). \\
\(p_j(x)\) & True probability that model \(j\) is correct on query \(x\). & \(p_j(x)=\Prob(Y_j=1\mid X=x)\). \\
\(\widehat p_j(x)\) & Router estimate of \(p_j(x)\); in the experiments it is a clipped average over 40 nearest router-training queries. & Query-specific marginal accuracy. \\
\(q_x(S)\), \(q(S)\) & Probability that every model in \(S\) is wrong, conditional on \(x\) or for the population. & All-miss probability. \\
\(c_x(S),c(S)\) & Probability that at least one model in \(S\) is correct. & Coverage; \(c_x(S)=1-q_x(S)\). \\
\(C^0_{S,x}\) & Correctness subcopula identified by the Bernoulli joint law on the values its marginal distributions can attain. & \(q_x(S)=C^0_{S,x}(1-p_j(x):j\in S)\). \\

\addlinespace
\multicolumn{3}{@{}l}{\textbf{Data and marginal estimates}}\\
\(D\in\{0,1\}^{n_d\times m}\) & Design correctness table; \(D_{ij}=1\) when model \(j\) answers design query \(i\) correctly. & \(n_d\) is the number of design queries. \\
\(F_{ij}=1-D_{ij}\) & Failure indicator for model \(j\) on design query \(i\). & Input to the dependence estimators. \\
\(\bar p_j^D\) & Fraction of design queries answered correctly by model \(j\). & \(\bar p_j^D=n_d^{-1}\sum_iD_{ij}\). \\
\(N(x)\), \(Z_{rj}\) & Fixed set of 40 nearest router-training queries for \(x\), and the correctness of model \(j\) on neighbor \(r\in N(x)\). & No design, calibration, validation, or test label enters this set. \\
\(\mathcal D_{\mathrm{des}},\mathcal D_{\mathrm{val}},\mathcal D_{\mathrm{cal}}\) & Design, validation, and final-calibration datasets. & Disjoint splits. \\
\(U\) & All auxiliary random choices frozen before final calibration. & Includes seeds, ties, smoothing, and numerical draws. \\

\addlinespace
\multicolumn{3}{@{}l}{\textbf{Ordering and prefixes}}\\
\(\pi,\pi_x\) & Ordering rule and the resulting permutation of models for query \(x\). & An order may be global or query-specific. \\
\(P_k^\pi(x)\) & Set containing the first \(k\) models in order \(\pi_x\). & \(P_0^\pi(x)=\varnothing\). \\
\(\pi^D\), \(\pi^T\) & Fixed order by decreasing design accuracy, and the fixed tally order of \eqref{eq:fixed-tally-order}. & Superscript \(T\) means tally, not transpose. \\
\(\widehat q_D(S)\) & Fraction of design queries on which all models in \(S\) are wrong. & \(n_d^{-1}\sum_i\prod_{j\in S}(1-D_{ij})\). \\
\(R_{k-1}\) & Design queries still unanswered after prefix \(P_{k-1}\). & A set of rows, not a correlation matrix. \\

\addlinespace
\multicolumn{3}{@{}l}{\textbf{Stopping scores}}\\
\(r_k(x)\) & Estimated all-miss residual assigned to the first \(k\) ordered models. & Generic sizing quantity; \(r_0(x)=1\). \\
\(r_k^{\mathrm{ind}}(x;\pi)\) & Independence-based all-miss estimate for the first \(k\) ordered models. & Product of query-specific marginal miss probabilities. \\
\(r_k^{\mathrm{dep}}(x;h,\pi)\) & Dependence-corrected all-miss estimate. & Adjusts \(r_k^{\mathrm{ind}}\) with the design-set dependence ratio; \(h=0\) recovers \(r_k^{\mathrm{ind}}\). \\
\(d_k,\widetilde d_k(h),h\) & Raw dependence ratio, its shrunk value, and the minimum-support parameter. & Defined in Appendix~\ref{app:implementation}. \\
\(r_k^{\mathrm{loc}}(x;\pi)\) & Fraction of the 40 router-training neighbors missed by every model in the first \(k\) positions. & Local empirical stopping. \\
\(s_\pi(x,\pi_{x,k}),s(x,j)\) & Entry score of the model at position \(k\), one minus the preceding residual. & Smaller scores enter first. \\
\(\lambda\), \(C_\lambda(x)\) & Cutoff on entry scores, and the model set it selects. & \(C_\lambda(x)=\{j:s(x,j)\leq\lambda\}\). \\

\addlinespace
\multicolumn{3}{@{}l}{\textbf{Calibration and certification}}\\
\(\alpha\) & User-specified target for the all-miss risk. & Reported targets range from 0.02 to 0.20. \\
\(n,X_i,G_i\) & Number of final-calibration queries, calibration query \(i\), and its set of correct models. & Exchangeable with a future query. \\
\(L^{\mathrm{miss}}(S,G)\) & Standard all-miss loss, equal to one when \(S\cap G=\varnothing\). & An empty called set has loss one. \\
\(\widehat R_n(\lambda),\widehat R_n(k)\) & Average calibration loss at a cutoff or depth. & Empirical all-miss risk. \\
\(\mathcal K_{\mathrm{cert}},\Lambda_{\mathrm{cert}}\) & Depths and cutoffs that pass the finite-sample calibration rule. & Empty means no certificate is issued. \\
\(\widehat k,\widehat\lambda\) & Smallest certified depth and smallest certified cutoff. & Call-minimal within the frozen family. \\
\(c\) & Number of calibration queries missed at a given cutoff. & The rule is \((c+1)/(n+1)\leq\alpha\). \\
\texttt{NON-CERTIFIED} & Output when no depth or cutoff passes calibration. & No policy is deployed under this certificate. \\
\(\mathrm{Success}@k\) & Fraction of queries for which at least one of the first \(k\) selected models is correct. & Fixed-depth coverage metric. \\
\end{longtable}
\endgroup

\section{Proofs for the joint target, ordering, and calibration}
\label{app:corrected-proofs}

\subsection{Formal joint target and independence benchmark}
\label{app:joint-target}

Let \(\cM=\{1,\ldots,m\}\) be the model pool. For a query \(X\), let \(Y_j=\ind\{\text{model }j\text{ is correct on }X\}\). For \(S\subseteq\cM\), recall
\[
q_x(S)=\Prob(Y_j=0\text{ for every }j\in S\mid X=x),
\qquad c_x(S)=1-q_x(S).
\]
The notation without subscript \(x\) denotes probabilities under one fixed population law.

\begin{proposition}[Identified correctness subcopula]
\label{prop:correctness-subcopula}
Fix \(x\) admitting a regular conditional law of \(Y=(Y_1,\ldots,Y_m)\), and let \(S\neq\varnothing\). Write \(p_j(x)=\Prob(Y_j=1\mid X=x)\). If \(F_{j,x}\) is the conditional CDF of \(Y_j\) and \(H_{S,x}\) is the conditional joint CDF of \((Y_j:j\in S)\), then a unique subcopula \(C^0_{S,x}\) on the attainable grid satisfies
\[
H_{S,x}(y_S)=C^0_{S,x}\bigl(F_{j,x}(y_j):j\in S\bigr).
\]
In particular,
\begin{equation}
q_x(S)=C^0_{S,x}\bigl(1-p_j(x):j\in S\bigr).
\label{eq:subcopula-value}
\end{equation}
The Bernoulli law identifies this subcopula only on \(\prod_{j\in S}\{0,1-p_j(x),1\}\), with repeated points removed. A full copula extension away from this grid is generally nonunique.
\end{proposition}

\begin{proof}
For a Bernoulli correctness variable,
\[
F_{j,x}(0)=\Prob(Y_j\leq0\mid X=x)=1-p_j(x).
\]
The discrete form of Sklar's theorem gives a unique subcopula on the Cartesian product of the marginal CDF ranges. Evaluating the joint CDF at the all-zero correctness vector gives
\[
q_x(S)
=H_{S,x}(0_S)
=C^0_{S,x}\bigl(F_{j,x}(0):j\in S\bigr)
=C^0_{S,x}\bigl(1-p_j(x):j\in S\bigr).
\]
For discrete margins, Sklar's theorem does not identify a unique extension outside the attainable grid.
\end{proof}

\paragraph{Orthant convention.}
Equation~\eqref{eq:subcopula-value} is a lower-orthant probability for the correctness vector. If \(B_j=1-Y_j\) denotes failure, the all-miss event is an upper-orthant probability for \(B\). It is not the lower-orthant copula value of the failure vector evaluated at its failure marginals.

\begin{lemma}[Individual-accuracy ordering under conditional independence]
\label{lem:conditional-independence-optimality}
Fix \(x\) and suppose \(Y_1,\ldots,Y_m\) are mutually independent conditional on \(X=x\). Assume equal call costs and one fixed tie rule. If \(\pi\) orders the models so that
\[
p_{\pi_1}(x)\geq p_{\pi_2}(x)\geq\cdots\geq p_{\pi_m}(x),
\]
then for every \(k\), the prefix \(P_k=\{\pi_1,\ldots,\pi_k\}\) minimizes \(q_x(S)\) over all model sets \(S\) with \(|S|=k\).
\end{lemma}

\begin{proof}
Conditional independence gives
\[
q_x(S)=\prod_{j\in S}\bigl(1-p_j(x)\bigr).
\]
Suppose a size-\(k\) set contains model \(a\) but excludes model \(b\), with \(p_b(x)>p_a(x)\). Replacing \(a\) by \(b\) replaces the factor \(1-p_a(x)\) by the no-larger factor \(1-p_b(x)\), so the all-miss probability cannot increase. Repeating this exchange until the set consists of the \(k\) largest marginals produces \(P_k\). The fixed tie rule makes these prefixes nested.
\end{proof}

Independence is sufficient but not necessary, since dependent outcomes may preserve individual-accuracy ordering or reverse it. Example~\ref{ex:nonnested-oracle} gives an exact finite law in which the individually best model is not part of the best pair.

\subsection{Ordering proof}
\label{app:ordering-proof}

Throughout this subsection, \(D\in\{0,1\}^{n_d\times m}\) is the realized design correctness matrix,
\[
G_i^D=\{j:D_{ij}=1\},
\qquad
U_D=\{i:G_i^D=\varnothing\},
\qquad
u_D=\frac{|U_D|}{n_d}=\widehat q_D(\cM).
\]
The design tally has the counting form
\begin{equation}
\widehat q_D(S)
=\frac{1}{n_d}\bigl|\{i:S\cap G_i^D=\varnothing\}\bigr|,
\qquad
\widehat q_D(\varnothing)=1.
\label{eq:tally-counting}
\end{equation}

\paragraph{Prefix-area identity.}
For an ordering \(\pi\), define
\[
T_i(\pi)=\bigl|\{k\in\{0,\ldots,m-1\}:P_k^\pi\cap G_i^D=\varnothing\}\bigr|.
\]
If design query \(i\) is answerable, \(T_i(\pi)\) is the position of its first correct model; if \(i\in U_D\), then \(T_i(\pi)=m\). Exchanging two finite sums yields
\begin{equation}
J_D(\pi)
=\sum_{k=0}^{m-1}\widehat q_D(P_k^\pi)
=\frac{1}{n_d}\sum_{i=1}^{n_d}T_i(\pi)
=\frac{1}{n_d}\sum_{i\notin U_D}T_i(\pi)+m\,u_D.
\label{eq:area-identity}
\end{equation}
The last term is independent of \(\pi\).

\paragraph{Reduction to greedy minimum-sum set cover.}
Let
\[
R_{k-1}=\{i:P_{k-1}\cap G_i^D=\varnothing\}
\]
be the design rows still unanswered after the current prefix. For any candidate model \(j\notin P_{k-1}\),
\begin{equation}
n_d\left[\widehat q_D(P_{k-1})-
\widehat q_D(P_{k-1}\cup\{j\})\right]
=\bigl|\{i\in R_{k-1}:D_{ij}=1\}\bigr|.
\label{eq:tally-residual-gain}
\end{equation}
Thus minimizing the next tally value is exactly the same as choosing the model that covers the most still-unanswered rows. Rows in \(U_D\) contribute zero to every candidate count. Once only those rows remain, every completion of the ordering has the same value of \(J_D\).

\begin{proof}[Proof of Theorem~\ref{thm:mssc-tight}]
Consider the set system whose elements are the answerable design rows \(i\notin U_D\) and whose set for model \(j\) is
\[
S_j=\{i\notin U_D:D_{ij}=1\}.
\]
By \eqref{eq:area-identity}, the first position containing a set that covers row \(i\) is \(T_i(\pi)\). Hence
\[
A(\pi)=\sum_{i\notin U_D}T_i(\pi)
\]
is exactly the minimum-sum set-cover objective. Equation~\eqref{eq:tally-residual-gain} shows that \(\pi^T\) is a greedy execution for this instance. The greedy algorithm is a factor-four approximation for minimum-sum set cover \citep{feige2004mssc}, so
\[
A(\pi^T)\leq4\min_\pi A(\pi).
\]
Substituting \(J_D(\pi)=A(\pi)/n_d+m\,u_D\) gives
\[
J_D(\pi^T)-m\,u_D
\leq4\left(\min_\pi J_D(\pi)-m\,u_D\right)
\]
and therefore \(J_D(\pi^T)\leq4\min_\pi J_D(\pi)\).

For hardness, start from an arbitrary minimum-sum set-cover instance in which every element belongs to at least one set. Define \(D_{ij}=1\) exactly when element \(i\) belongs to set \(j\). Then \(U_D=\varnothing\), and \(A(\pi)\) is the objective of the original instance. A polynomial-time algorithm with ratio \(4-\varepsilon\) for every correctness matrix would therefore give the same ratio for every minimum-sum set-cover instance, which is impossible for any \(\varepsilon>0\) unless \(\mathrm P=\mathrm{NP}\) \citep{feige2004mssc}.
\end{proof}

\begin{theorem}[Fixed-budget coverage guarantee]
\label{thm:fixed-cardinality-submodularity}
For any fixed law of the correctness indicators, the coverage
\[
c(S)=\Prob\!\left(\bigcup_{j\in S}\{Y_j=1\}\right)
\]
is normalized, monotone, and submodular. Let \(G_0=\varnothing\), and repeatedly add the model with the largest true coverage gain. For every fixed \(k\geq1\), the resulting set \(G_k\) satisfies
\begin{equation}
c(G_k)
\geq\left(1-\left(1-\frac{1}{k}\right)^k\right)
\max_{|S|=k}c(S)
\geq\left(1-\frac{1}{e}\right)\max_{|S|=k}c(S).
\label{eq:design-per-depth}
\end{equation}
The same statement holds for \(c_x\) after fixing \(x\).
\end{theorem}

\begin{proof}
For \(A\subseteq B\) and \(j\notin B\),
\[
c(A\cup\{j\})-c(A)
=\Prob(Y_j=1,\ Y_i=0\text{ for every }i\in A),
\]
while the analogous event with \(B\) is a subset of this event. This proves diminishing returns and therefore submodularity. Let \(O_k\) maximize coverage among size-\(k\) sets and write \(\mathrm{OPT}=c(O_k)\). By monotonicity and submodularity,
\[
\mathrm{OPT}
\leq c(G_t)+
\sum_{j\in O_k\setminus G_t}
\bigl[c(G_t\cup\{j\})-c(G_t)\bigr].
\]
There are at most \(k\) terms, so the next greedy gain is at least \((\mathrm{OPT}-c(G_t))/k\). Iterating from \(c(G_0)=0\) for \(k\) steps gives the first inequality in \eqref{eq:design-per-depth}; the second uses \((1-1/k)^k\leq e^{-1}\).
\end{proof}

Under the uniform law on design rows, the true greedy step in Theorem~\ref{thm:fixed-cardinality-submodularity} is exactly the tally step in \eqref{eq:tally-residual-gain}. Thus the same fixed tally ordering has the per-depth guarantee \eqref{eq:design-per-depth} and the whole-order guarantee in Theorem~\ref{thm:mssc-tight}. The first compares one depth with the best set at that depth; the second compares the area of one complete ordering with the best complete ordering.

\subsection{Calibration proof}
\label{app:calibration-proof}

Let \(\mathcal D_{\mathrm{des}}\) and \(\mathcal D_{\mathrm{val}}\) be the design and validation data, and let \(U\) contain every auxiliary random choice, including seeds, numerical draws, clustering choices, tie rules, smoothing, and candidate-family selection. Before final-calibration labels are inspected, these quantities determine one measurable nested map
\[
(x,k)\longmapsto C_k(x;\mathcal D_{\mathrm{des}},\mathcal D_{\mathrm{val}},U)
\subseteq\cM,
\qquad k=0,\ldots,K.
\]
Conditioning on the frozen information leaves the calibration observations and future observation exchangeable.

\begin{proof}[Proof of Theorem~\ref{thm:frozen-feasible-crc}]
For observation \(i\), define the first-success depth
\[
D_i=\min\{k\in\{0,\ldots,K\}:L_i(k)=0\},
\]
with \(D_i=K+1\) if no set in the family succeeds. Nestedness gives \(L_i(k)=\ind\{D_i>k\}\), and \(D_1,\ldots,D_{n+1}\) are exchangeable conditional on the frozen information.

If \(\alpha<1/(n+1)\), no depth can pass and the certified-miss event is empty. Otherwise, set
\[
r=\lfloor\alpha(n+1)\rfloor-1.
\]
A depth \(k\) is certified exactly when
\[
\sum_{i=1}^n\ind\{D_i>k\}\leq r.
\]
Let \(D_{(1)}\leq\cdots\leq D_{(n)}\) be the calibration order statistics. A certificate is issued exactly when \(D_{(n-r)}\leq K\), and then \(\widehat k=D_{(n-r)}\). A future miss implies \(D_{n+1}>D_{(n-r)}\), so the future depth occupies one of the largest \(r+1\) ranks among the \(n+1\) exchangeable depths. Therefore
\[
\Prob\!\left(
\mathcal K_{\mathrm{cert}}\neq\varnothing
\text{ and }L_{n+1}(\widehat k)=1
\mid \mathcal D_{\mathrm{des}},\mathcal D_{\mathrm{val}},U
\right)
\leq\frac{r+1}{n+1}\leq\alpha.
\]
Averaging over the frozen information proves the unconditional claim. If certification occurs almost surely, the certificate event can be removed.
\end{proof}

\begin{proof}[Proof of Corollary~\ref{cor:minimal-cutoff}]
Nestedness makes \(\widehat R_n(\lambda)\) nonincreasing, so \(\Lambda_{\mathrm{cert}}\) is an upper set. The finite calibration sample generates at most \(nm\) distinct scores,
\[
\mathcal S_{\mathrm{cal}}
=\{s(X_i,j):1\leq i\leq n,\ j\in\cM\}.
\]
Hence \(\widehat R_n\) is a right-continuous step function whose jumps occur only at values in \(\mathcal S_{\mathrm{cal}}\). The infimum of a nonempty \(\Lambda_{\mathrm{cert}}\) is therefore attained. Every certified \(\lambda\) satisfies \(\lambda\geq\widehat\lambda\), so
\[
C_{\widehat\lambda}(x)\subseteq C_\lambda(x)
\]
for every query \(x\). Applying any cost that is nondecreasing under set inclusion preserves this inequality, both pointwise and in expectation.
\end{proof}

\paragraph{Continuous-cutoff form.}
The score rule \eqref{eq:crc} also reduces to a one-dimensional conformal statement. For calibration query \(i\), let \(S_i^*=\min_{j\in G_i}s(X_i,j)\), with \(S_i^*=+\infty\) when \(G_i=\varnothing\). The score map is frozen before calibration, so \(S_1^*,\ldots,S_{n+1}^*\) are exchangeable, and the all-miss loss of \(C_\lambda\) on query \(i\) equals \(\ind\{S_i^*>\lambda\}\). With \(r=\lfloor\alpha(n+1)\rfloor-1\) as in Lemma~\ref{lem:rank-form}, the rule \eqref{eq:crc} returns \(\widehat\lambda=S_{(n-r)}^*\), the \((n-r)\)th order statistic of the calibration scores, and certifies exactly when this value is finite. A certified miss on the next query is the event \(\{S_{n+1}^*>S_{(n-r)}^*\}\), whose probability is at most \((r+1)/(n+1)\leq\alpha\) by exchangeability.

\subsubsection{Exact rank form and sharpness}

\begin{lemma}[Exact rank form]
\label{lem:rank-form}
Let \(D_i\) be the first-success depths defined in the proof of Theorem~\ref{thm:frozen-feasible-crc} and set \(r=\lfloor\alpha(n+1)\rfloor-1\). When \(1/(n+1)\leq\alpha<1\), depth \(k\) is certified if and only if \(|\{i\leq n:D_i>k\}|\leq r\). A certificate is issued if and only if \(D_{(n-r)}\leq K\), and then \(\widehat k=D_{(n-r)}\). Moreover,
\begin{equation}
\Prob\bigl(\mathcal K_{\mathrm{cert}}\neq\varnothing
\text{ and }L_{n+1}(\widehat k)=1\bigr)
\leq
\frac{\lfloor\alpha(n+1)\rfloor}{n+1}
\leq\alpha.
\label{eq:rank-form}
\end{equation}
\end{lemma}

\begin{proof}
Write \(c_k=|\{i\leq n:L_i(k)=1\}|\). Since \(c_k\) is an integer,
\[
\frac{n}{n+1}\frac{c_k}{n}+\frac{1}{n+1}\leq\alpha
\quad\Longleftrightarrow\quad
c_k\leq\lfloor\alpha(n+1)\rfloor-1=r.
\]
Because \(L_i(k)=\ind\{D_i>k\}\), this is equivalent to \(D_{(n-r)}\leq k\). The remaining claims follow from the proof of Theorem~\ref{thm:frozen-feasible-crc}.
\end{proof}

\begin{theorem}[Sharp deterministic rank correction]
\label{thm:rank-sharpness}
Fix \(n\geq1\) and \(1/(n+1)\leq\alpha<1\), and let \(r=\lfloor\alpha(n+1)\rfloor-1\).
\begin{enumerate}
\item There is an exchangeable one-model problem for which the probability in \eqref{eq:rank-form} equals \(\lfloor\alpha(n+1)\rfloor/(n+1)\).
\item If a deterministic rule accepts one more calibration miss than \(r\), there is an exchangeable one-model problem on which its certified-miss probability exceeds \(\alpha\).
\end{enumerate}
Thus the additive correction \(1/(n+1)\) in \eqref{eq:crc} is the smallest valid correction for deterministic distribution-free rank calibration.
\end{theorem}

\begin{proof}
Take \(m=1\), \(K=1\), \(C_0(x)=\varnothing\), and \(C_1(x)=\{1\}\). Let \(Y^{(1)},\ldots,Y^{(n+1)}\) be a uniformly random permutation of a fixed binary vector with exactly \(z\) zeros. This law is exchangeable. The first-success depth is \(1\) when \(Y^{(i)}=1\) and \(2\) otherwise.

For the first claim, choose \(z=r+1\). Depth zero calls no model and therefore misses every calibration query, so it never passes calibration, whereas depth one passes calibration exactly when the future observation is one of the \(z\) zeros, and on that event the certified set misses. The probability is therefore
\[
\frac{r+1}{n+1}
=\frac{\lfloor\alpha(n+1)\rfloor}{n+1}.
\]
For the second claim, suppose a rule accepts at most \(t\geq r+1\) calibration misses. If \(t\geq n\), it accepts the empty set and misses with probability one. Otherwise choose \(z=t+1\). The same construction certifies exactly when the future observation is a zero, giving probability
\[
\frac{t+1}{n+1}
\geq\frac{r+2}{n+1}
>\alpha.
\]
Therefore no deterministic rule can accept more misses while retaining the guarantee for every exchangeable law. Replacing the additive one by any constant \(b<1\) admits such an extra miss at some attainable target and fails by the same construction.
\end{proof}

\begin{corollary}[Certification floor]
\label{cor:certification-floor}
If \(\alpha<1/(n+1)\), no depth can pass the calibration rule, even with zero observed calibration misses. For \(\alpha\geq1/(n+1)\), the certified level is \(\lfloor\alpha(n+1)\rfloor/(n+1)\), so the attainable deterministic levels are integer multiples of \(1/(n+1)\).
\end{corollary}

\begin{proof}
The corrected empirical risk is always at least \(1/(n+1)\), which proves the first claim. Lemma~\ref{lem:rank-form} gives the second claim, and Theorem~\ref{thm:rank-sharpness} shows that the bound is attained.
\end{proof}



\section{Oracle, scope, and structured-estimator theory}
\label{app:extended-theory}

\subsection{Oracle and scope results}

For one fixed query or one fixed population law, define the best all-miss probability attainable by any size-\(k\) model set as
\begin{equation}
q_k^*=\min_{S\subseteq\cM:\,|S|=k}q(S).
\label{eq:cardinality-oracle}
\end{equation}
These values are useful lower bounds, but independently optimal sets at different sizes do not automatically form one routing order.

\begin{proposition}[Cardinality-oracle sets need not be nested]
\label{prop:nonnested-cardinality-oracle}
For \(k=0,\ldots,m-1\), \(q_{k+1}^*\leq q_k^*\), and every size-\(k\) set \(T\) satisfies \(q(T)\geq q_k^*\). Nevertheless, an optimizer at size \(k+1\) need not contain any optimizer at size \(k\).
\end{proposition}

\begin{proof}
Choose an optimizer \(S_k^*\) of size \(k<m\) and any model \(j\notin S_k^*\). The event that every model in \(S_k^*\cup\{j\}\) fails is contained in the event that every model in \(S_k^*\) fails. Hence
\[
q_{k+1}^*
\leq q(S_k^*\cup\{j\})
\leq q(S_k^*)
=q_k^*.
\]
The lower bound for \(T\) follows from the definition of the minimum. Example~\ref{ex:nonnested-oracle} proves the last claim.
\end{proof}

\begin{example}[The best singleton and pair need not nest]
\label{ex:nonnested-oracle}
Consider three models and four population atoms whose probabilities and failing models are
\[
\begin{array}{c|c|c}
\text{atom}&\text{probability}&\text{failing models}\\
\hline
a&0.2&\{1,2\}\\
b&0.3&\{2\}\\
c&0.2&\{1,3\}\\
d&0.3&\{3\}.
\end{array}
\]
The singleton miss probabilities are \(q(\{1\})=0.4\) and \(q(\{2\})=q(\{3\})=0.5\), so \(\{1\}\) is the unique best singleton. The pair miss probabilities are
\[
q(\{1,2\})=q(\{1,3\})=0.2,
\qquad q(\{2,3\})=0,
\]
so \(\{2,3\}\) is the unique best pair, and the two optimizers are not nested.
\end{example}

\begin{proposition}[Target-specific oracle headroom]
\label{prop:alpha-specific-headroom}
Let \((M_k)_{k=0}^m\) be a fixed marginal-accuracy prefix chain under the same population set function \(q\), and assume that the stopping depths
\begin{equation}
k_\alpha^*=\min\{k:q_k^*\leq\alpha\},
\qquad
k_\alpha^{\mathrm{marg}}
=\min\{k:q(M_k)\leq\alpha\}
\label{eq:alpha-stopping-depths}
\end{equation}
exist. Then
\begin{equation}
k_\alpha^*<k_\alpha^{\mathrm{marg}}
\quad\Longleftrightarrow\quad
\exists S\subseteq\cM:
|S|<k_\alpha^{\mathrm{marg}}
\text{ and }q(S)\leq\alpha.
\label{eq:alpha-specific-improvement}
\end{equation}
Any optimizer \(S_\alpha^*\) of size \(k_\alpha^*\) can be placed first in a target-specific nested order whose first successful prefix has size \(k_\alpha^*\).
\end{proposition}

\begin{proof}
If \(k_\alpha^*<k_\alpha^{\mathrm{marg}}\), choose an optimizer \(S_\alpha^*\) at size \(k_\alpha^*\). It satisfies the right side of \eqref{eq:alpha-specific-improvement}. Conversely, if a set \(S\) of size \(\ell<k_\alpha^{\mathrm{marg}}\) has \(q(S)\leq\alpha\), then
\[
q_\ell^*\leq q(S)\leq\alpha,
\]
so \(k_\alpha^*\leq\ell<k_\alpha^{\mathrm{marg}}\). For the final claim, order the elements of \(S_\alpha^*\) first and append all remaining models. No smaller prefix can meet the target, because every set of smaller size has all-miss probability greater than \(\alpha\).
\end{proof}

This proposition is an existence result for one target, and a different target may require a different subset and order. It does not construct an estimated policy or one chain that attains every \(q_k^*\).

\paragraph{Cost-aware oracle and heuristic.}
If model \(j\) has positive call cost \(a_j\), the exact population reference at target \(\alpha\) is
\begin{equation}
S_{\mathrm{cost}}^*(\alpha)
\in\arg\min_{S\subseteq\cM}
\sum_{j\in S}a_j
\quad\text{subject to}\quad q(S)\leq\alpha.
\label{eq:exact-cost-oracle}
\end{equation}
Its optimizers are in general not nested, and a fixed order may instead append
\begin{equation}
j^*(S;x)
\in\arg\max_{j\notin S}
\frac{\widehat q_x(S)-\widehat q_x(S\cup\{j\})}{a_j}.
\label{eq:gain-per-cost-heuristic}
\end{equation}
This is a greedy cost heuristic and does not minimize cost. Once its family is selected without final-calibration labels and then frozen, Theorem~\ref{thm:frozen-feasible-crc} can certify its all-miss risk. No cost-approximation ratio follows from that certification.

\paragraph{Matching an ordering does not imply that one policy's sets contain the other's.}
Consider one model, two query values \(x_a,x_b\), and a score threshold \(\tau\) such that
\[
s_1(x_a)<\tau<s_1(x_b).
\]
A threshold policy with a null action may select
\[
A_\tau(x_a)=\{1\},
\qquad
A_\tau(x_b)=\varnothing.
\]
A policy that instead calibrates one common depth calls no model at depth zero and at least one model at every positive depth, so it cannot reproduce both choices. Thus matching a model order under independence does not establish that the sets of a common-depth policy always contain, or are always contained in, those of a query-dependent threshold-and-null policy. Whether one policy's sets contain the other's also depends on the score, the cutoff, the null action, the smoothing, and the tie behavior.

\subsection{Population transfer of the ordering guarantee}
\label{app:ordering-transfer}

Theorem~\ref{thm:mssc-tight} is a statement about the realized design matrix, and this subsection records when it transfers to the deployment distribution and why the transferred bound is uninformative at large pools. Write \(J_P(\pi)=\E[T_X(\pi)]\) for the population average position of the first correct model under ordering \(\pi\), with the same truncation convention as \(J_D\).

\begin{proposition}[Transfer with an additive sampling term]
\label{prop:ordering-transfer}
Suppose the \(n_d\) design queries are drawn independently from the deployment distribution, and let \(\delta\in(0,1)\). With probability at least \(1-\delta\) over the design draw, the tally order \(\pi^T\) satisfies
\[
J_P(\pi^T)\leq 4\min_\pi J_P(\pi)+\varepsilon_u+4\varepsilon_p,
\qquad
\varepsilon_u=m\sqrt{\tfrac{m\log m+\log(4/\delta)}{2n_d}},
\quad
\varepsilon_p=m\sqrt{\tfrac{\log(4/\delta)}{2n_d}}.
\]
\end{proposition}

\begin{proof}
For a fixed ordering \(\pi\), \(J_D(\pi)\) is an average of \(n_d\) independent terms \(T_i(\pi)\in[1,m]\), so Hoeffding's inequality gives \(|J_D(\pi)-J_P(\pi)|\leq\varepsilon_p\) with probability at least \(1-\delta/2\) for one fixed ordering, and \(|J_D(\pi)-J_P(\pi)|\leq\varepsilon_u\) simultaneously for all \(m!\) orderings with probability at least \(1-\delta/2\), using \(\log m!\leq m\log m\). On the intersection of the two events, with \(\pi^\star\) a population minimizer,
\[
J_P(\pi^T)\leq J_D(\pi^T)+\varepsilon_u
\leq 4\min_\pi J_D(\pi)+\varepsilon_u
\leq 4J_D(\pi^\star)+\varepsilon_u
\leq 4J_P(\pi^\star)+4\varepsilon_p+\varepsilon_u,
\]
where the second step is Theorem~\ref{thm:mssc-tight}. The uniform bound is needed for \(\pi^T\) because the tally order is a function of the design draw.
\end{proof}

Two features limit this statement. First, the additive term decides whether the bound is informative. At the eleven-model pools, whose design splits contain between about eight thousand and eighty thousand queries, \(\varepsilon_u+4\varepsilon_p\) stays below one and a half calls at \(\delta=0.05\), so the transferred factor-four bound is informative there. At one hundred models the term already reaches tens of calls, and at five hundred models it is comparable to the pool size and far above the certified call counts, so the bound is vacuous where the pools are largest, and the redraw protocol of Section~\ref{sec:experiments} measures design-to-test generalization directly in that regime. Second, even where the bound transfers, it controls the mean of the first-success depth, while calibration selects, in the depth-indexed view, a quantile of that depth near the target level. A mean-competitive ordering can still cross the target level late, so no multiplicative guarantee on certified calls at a fixed target follows, beyond Markov-type bounds whose factor inflates by \(1/\alpha\). The number of calls that CopRoute makes on test queries is therefore an empirical finding in this paper, and Theorem~\ref{thm:mssc-tight} does not imply it; the guarantee that holds at test time is the certificate of Theorem~\ref{thm:frozen-feasible-crc}.

\subsection{Structured-estimator theory}

A parametric dependence model should describe association left after the marginal router has accounted for query difficulty, whereas raw pooled correlation can reflect a mixture of easy and hard queries and may count the same difficulty twice when combined with query-specific marginals.

The rest of this subsection concerns the parametric members of the family, namely the Gaussian and Student-t copula and latent-normal orders of Appendix~\ref{app:estimators}, and does not concern the tally default. One structured working model is conditional multivariate probit, in which a latent Gaussian vector with query-dependent correlation is thresholded at the levels that reproduce the router's marginals, namely
\begin{equation}
Z\mid X=x\sim\mathcal N(0,R_\theta(x)),
\qquad
Y_j=\ind\{Z_j>t_j(x)\},
\qquad
t_j(x)=\Phi^{-1}(1-p_j(x)).
\label{eq:latent-gaussian-working-model}
\end{equation}
Under correct specification,
\begin{equation}
q_x(S)=\Phi_{R_{\theta,S}(x)}\bigl(t_S(x)\bigr).
\label{eq:model-based-orthant}
\end{equation}
This is a model-based formula and is not an identity for every Bernoulli law, and the Gaussian, Student \(t\), vine, Ising, and other extensions are modeling choices.

\begin{assumption}[Conditional residual dependence]
\label{ass:conditional-residual-dependence}
The structured estimator is fitted on design data under the following conditions.
\begin{enumerate}
\item Each marginal \(p_j(x)\) used to fit dependence is estimated out of fold.
\item The dependence estimator targets the residual law after accounting for the marginal probabilities and declared router features.
\item Using a shared matrix \(R\) assumes that residual dependence is constant across queries, and clustered or feature-dependent versions are learned on design data only.
\item Marginal probabilities are clipped away from zero and one, and fitted correlation matrices are positive definite with eigenvalues bounded away from zero.
\item The dependence family, numerical integration rule, and tuning choices are selected without final-calibration labels.
\end{enumerate}
\end{assumption}

\begin{theorem}[Structured plug-in consistency for fixed \(m\)]
\label{thm:fixed-m-plugin-consistency}
Keep \(m\) fixed while the number of design queries \(n_d\) grows. Assume the design observations are iid, Assumption~\ref{ass:conditional-residual-dependence} holds, and for a fresh query \(X\),
\begin{equation}
\max_{1\leq j\leq m}
|\widehat p_j(X)-p_j^\dagger(X)|\xrightarrow{p}0,
\qquad
\widehat\theta\xrightarrow{p}\theta^\dagger.
\label{eq:parameter-consistency}
\end{equation}
Also assume a compact parameter space, correlation eigenvalues bounded away from zero, finite thresholds, continuity of the orthant map and any correlation projection at the limit, and numerical integration error converging uniformly to zero over subsets. Let \(q_X^\dagger(S)\) be the all-miss probability under the limiting structured model. Then
\begin{equation}
\max_{S\subseteq\cM}
|\widehat q_X(S)-q_X^\dagger(S)|\xrightarrow{p}0.
\label{eq:fixed-m-uniform-plugin}
\end{equation}
Moreover,
\begin{equation}
\max_{S\subseteq\cM}|\widehat q_X(S)-q_X(S)|
\leq
\max_{S\subseteq\cM}|\widehat q_X(S)-q_X^\dagger(S)|
+b_m(X),
\label{eq:plugin-bias-decomposition}
\end{equation}
where
\begin{equation}
b_m(X)=\max_{S\subseteq\cM}|q_X^\dagger(S)-q_X(S)|
\label{eq:misspecification-bias}
\end{equation}
is the misspecification error. Under correct specification, \(b_m(X)=0\).
\end{theorem}

\begin{proof}
Fix a model set \(S\). Marginal consistency, parameter consistency, continuity, and vanishing numerical error imply
\[
\widehat q_X(S)-q_X^\dagger(S)\xrightarrow{p}0
\]
by the continuous mapping theorem. Because \(m\) is fixed, there are only \(2^m\) subsets. A union bound over this finite collection gives the uniform convergence in \eqref{eq:fixed-m-uniform-plugin}. Equation~\eqref{eq:plugin-bias-decomposition} is the triangle inequality. Correct specification identifies \(q_X^\dagger\) with the true all-miss law and makes \(b_m(X)=0\).
\end{proof}

This theorem concerns the parametric estimators only, and it states that a correctly specified parametric model of the joint law, fitted consistently, recovers the true all-miss probabilities of every set in the limit. It says nothing about the tally order, which fits no parametric model; it gives no sample-complexity advantage of parametric over empirical estimators, and it gives no guarantee when the parametric model is misspecified. Decision convergence additionally requires a unique limiting winner separated by a positive margin, or a vanishing probability of near ties.

\paragraph{Scope summary.}
The appendix establishes the formal joint target, the independence benchmark, proofs of the fixed-tally and calibration guarantees, oracle lower bounds, and one structured-model consistency result. It does not establish one universally optimal nested chain, strict containment of the method of \citet{racer2026}, population transfer of the empirical factor-four constant, optimality across different frozen families, conditional risk for every query, final-answer correctness, or a structured small-sample advantage.

\section{Implementation details}
\label{app:implementation}

\textbf{Frozen data and marginal router.}
The design correctness matrix is \(D\in\{0,1\}^{n_d\times m}\), with design marginal \(\bar p_j^D=n_d^{-1}\sum_iD_{ij}\). The design split is disjoint from router training, validation, final calibration, and test. The frozen prompt representations are the precomputed 768-dimensional mDeBERTa embeddings distributed with the benchmark releases, deduplicated and row-wise L2 normalized. No encoder is trained or loaded in our implementation, and nearest neighbors use cosine similarity. For query \(x\), let \(N(x)\) be its 40 nearest router-training queries under this embedding, and let \(Z_{rj}\) indicate whether model \(j\) answers neighbor \(r\) correctly. The marginal router uses \(\widehat p_j(x)=|N(x)|^{-1}\sum_{r\in N(x)}Z_{rj}\), clipped to \([10^{-6},1-10^{-6}]\). No design, validation, calibration, or test label enters \(N(x)\).

\textbf{Fixed-tally construction.}
The main order \(\pi^T\) follows \eqref{eq:fixed-tally-order}. An exact tally tie is resolved by the larger \(\bar p_j^D\) and then by the smaller model index. If the current prefix covers every design query, the remaining models are appended by descending \(\bar p_j^D\) and then by model index. This produces one complete order per fitted benchmark and model pool.

\textbf{Independence-based stopping.}
For any order \(\pi\), the independence residual is \(r_k^{\mathrm{ind}}(x;\pi)=\prod_{\ell=1}^k(1-\widehat p_{\pi_{x,\ell}}(x))\), with \(r_0^{\mathrm{ind}}=1\). The model at position \(k\) receives the predecessor-coverage score \(1-r_{k-1}^{\mathrm{ind}}\). We use conditional independence only to estimate prefix failure, and final calibration does not assume that this estimate equals the true all-miss probability.

\textbf{Dependence-based stopping.}
For the fixed-tally order, let \(a_k=\widehat q_D(P_k^{\pi^T})\) be the observed design all-miss fraction and \(b_k=\prod_{\ell=1}^k(1-\bar p_{\pi^T_\ell}^D)\) its independence prediction. Their ratio \(d_k=a_k/b_k\) is the design-set dependence ratio. We set \(d_0=1\), and we set \(d_k=1\) whenever \(b_k=0\), which occurs only if some model in the prefix answers every design query, and then \(a_k=0\) as well, so the independence product already returns zero. Let \(c_k=n_da_k\) be the number of design queries missed by the entire prefix. For \(h>0\), define
\begin{equation}
 w_k(h)=\min\{1,c_k/h\},
 \qquad
 \widetilde d_k(h)=1+w_k(h)(d_k-1).
\label{eq:dependence-ratio-shrinkage}
\end{equation}
Prefixes supported by at least \(h\) surviving design queries use the full ratio; prefixes with fewer supporting queries move the ratio toward one. Validation considers \(h\in\{0,5,20,50\}\). The value \(h=0\) is defined by \(\widetilde d_k(0)=1\) at every depth and therefore recovers independence-based stopping exactly. For each query,
\begin{equation}
 r_k^{\mathrm{dep}}(x;h)=
 \min_{0\leq t\leq k}
 \operatorname{clip}_{[0,1]}\!\left\{
 r_t^{\mathrm{ind}}(x;\pi^T)\widetilde d_t(h)
 \right\}.
\label{eq:dependence-stopping-full}
\end{equation}
The cumulative minimum keeps the residual from increasing with depth, so thresholding the entry scores always selects an exact prefix. The design-set dependence ratio is an empirical correction along this order and does not estimate the full correctness copula parametrically.

\textbf{Query-specific and local constructions.}
An order or selected depth is query-specific when it changes with \(x\), and a dependence estimate is local when it is fitted from \(N(x)\). Every local method is query-specific, while pairwise and correlation methods can be query-specific even when they combine \(\widehat p_j(x)\) with one dependence table fitted on the global design split. The two main policies use a fixed order and a query-specific selected depth.

For a query-specific order \(\pi_x\), local empirical stopping counts joint misses among the 40 neighbors as
\begin{equation}
 r_k^{\mathrm{loc}}(x;\pi)=
 \frac{1}{40}\sum_{r\in N(x)}
 \prod_{\ell=1}^{k}(1-Z_{r,\pi_{x,\ell}}),
 \qquad r_0^{\mathrm{loc}}=1.
\label{eq:local-sizing}
\end{equation}
The fully local policy greedily appends the candidate that minimizes this residual and also uses it for stopping. A matched variant keeps the same local order but uses \(r_k^{\mathrm{ind}}\) for stopping. Residuals within \(10^{-12}\) of one another are treated as ties, which are resolved by larger local marginal correctness and then by smaller model index.

\textbf{Ties and numerical audit.}
Before calibration, every query-level score receives the same frozen perturbation used by RACER, namely independent noise drawn uniformly from \([0,10^{-6}]\), keyed by the experiment and split seeds and added to the raw predecessor-coverage score of each model. The perturbation resolves equal or nearly equal cutoff scores so that the sublevel sets form a nested family in the threshold, and it changes the deployed order only among models whose raw scores are within \(10^{-6}\) of one another. Such near-ties arise when the running product has already reached zero, so that every remaining model in the fixed order carries a score of one. On such a plateau the independent noise can admit a model at a later position ahead of an earlier one, and the certified set is then not a prefix of the fixed order at that depth. Because the plateau begins only after the estimated all-miss probability is zero, calibration treats every model on it as interchangeable, and the guarantee of Theorem~\ref{thm:frozen-feasible-crc} covers the perturbed score map as it was run. We record for each cell the depth at which the plateau begins, the fraction of test queries whose calibrated set reaches it, and the corresponding underflow and saturation flags. A prefix-preserving alternative, which adds the offset \((k-1)10^{-6}/m\) at position \(k\) of the stably sorted scores, is implemented but was not used for the reported runs. For independence-based stopping, we retain \(\log r_k^{\mathrm{ind}}=\sum_{\ell=1}^k\log(1-\widehat p_{\pi_{x,\ell}}(x))\) beside the direct product.

\textbf{Learned baseline reconstructions.}
FlexRouter uses a trainable determinantal point process with a 128-dimensional query projection, trainable model embeddings, a cosine-similarity kernel, a coverage objective with auxiliary binary cross-entropy, and greedy determinant-gain ordering \citep{flexrouter2026}. Because its exact encoder checkpoint and pooling implementation are not released with executable code, our reconstruction uses the same audited frozen 768-dimensional prompt embeddings that the other learned baselines receive and sets the unreported pre-screen size to \(K_0=m\).

BaRP uses its preference-conditioned categorical policy and paper-specified REINFORCE training with chosen-action feedback, a mini-batch reward baseline, entropy regularization, and uniformly sampled quality-cost preferences \citep{wang2025barp}. Its unreported preference width, decision-head width, and cost cap are frozen at 32, 256, and the 95th percentile of finite positive training-family call costs, respectively. For common set evaluation, its quality-only logits induce a stable top-\(k\) order. Both learned baselines use five training seeds, and in the common set-routing evaluation their orders extend over the full model pool.

\textbf{Computational cost.} All policies read the same router scores, so the shared per-query cost is one embedding lookup and an average over the 40 stored neighbors, and the table below lists what each policy adds. Fitting the tally order is a counting computation on the design matrix and involves no gradient training.

\begin{center}
\scriptsize
\setlength{\tabcolsep}{4pt}
\begin{tabular}{lp{4.6cm}p{4.0cm}l}
\toprule
policy & fitting, offline & added per-query cost & tuned \\
\midrule
CopRoute-Ind & greedy tally on the design matrix, $O(n_d m)$ per selected position & running product along the fixed order, $O(m)$ & none \\
CopRoute-Dep & tally plus prefix co-failure counts, $O(n_d m)$ & product times stored ratios and a running minimum, $O(m)$ & $h$ over four values \\
RACER-G / RACER-P & none beyond the router & per-model score transform, $O(m)$ & none \\
Exact independence & none & per-query sort of the marginals, $O(m\log m)$ & none \\
FlexRouter & gradient training of the query projection and model embeddings, per pool and seed & encoder projection and greedy determinant-gain selection over the pool & training settings \\
BaRP & policy-gradient training, per pool and seed & network forward pass and top-$k$ sort, $O(m\log m)$ & training settings \\
\bottomrule
\end{tabular}
\end{center}

The proposed policies fit a single greedy pass over the design matrix. This is more than RACER fits, since RACER fits nothing beyond the router, and less than the learned baselines fit, since they require gradient training for every pool and seed before any routing can occur. At deployment, every single-router policy adds one linear pass over the ordered pool per query.

\textbf{Selection before final calibration.}
Each repetition fits the router, order, and dependence objects on their designated training splits. We tune each hyperparameter on the validation split and select a family there when required. We then freeze the score map, real-model sets, tie rules, seeds, and numerical draws. Held-out final-calibration labels select only the global cutoff through \eqref{eq:crc}, after which the test split is evaluated once. Appendix~\ref{app:protocol} gives the split proportions and feasibility rule.

\section{Estimator definitions for the CopRoute family}
\label{app:estimators}

The query-level router provides \(\widehat p_j(x)\), and the design matrix provides the failure indicators \(F_{ij}=1-D_{ij}\). Every fixed order in the family is produced by the same greedy step, which repeatedly adds the model that is wrong least often on the design queries where every already-chosen model is wrong. The estimators differ only in the failure matrix fed to that step, either the observed design failures or synthetic failures drawn from a joint model fitted to them. The names below match Tables~\ref{tab:estimator-family}, \ref{tab:estimator-grid}, \ref{tab:estimator-by-dataset}, and~\ref{tab:mechanism-map}.

\textbf{Tally (the default).} The tally order applies the greedy step directly to the observed design failure matrix and is the fixed tally order of Section~\ref{sec:method}.

\textbf{Double fault.} This pairwise variant of the same idea adds the model whose joint-failure rate with the already-chosen set is smallest, using a weight that trades this diversity term against individual accuracy.

\textbf{Pairwise count.} This order estimates pairwise joint failures with additive smoothing, averages the implied conditional-miss surrogates over the chosen set, and multiplies the running residual by that average. It is a bounded ranking heuristic rather than a coherent joint probability, and it appears in the synthetic mechanism map (Table~\ref{tab:mechanism-map}).

\textbf{Error correlation.} This order adds the model whose failures are least Pearson-correlated with the chosen set, with the correlation magnitude clipped so that a near-perfect estimate from a small design split is not taken as exact.

\textbf{Rank correlation.} This order uses the latent-threshold construction below, with the correlation estimated from failure concordance through the sine transform used for rank correlations, so that the estimate is unaffected by miscalibration of the marginals.

\textbf{Latent threshold models (Gaussian copula, Student-t copula, low-rank factor copula).} Each model fails when a latent variable falls below a threshold matched to its design failure rate. Dependence is carried by a low-rank factor approximation of the failure correlation, computed from an economy SVD of the standardized failures, since a full correlation matrix over 3{,}811 models can neither be formed nor supported by a few hundred design labels. The Gaussian variant draws the latent variables as normals (the mechanism map calls this construction latent normal), the Student-t variant gives them heavier tails, and the low-rank factor variant fixes the rank. Synthetic failures drawn from the fitted model replace the observed matrix in the greedy step.

\textbf{Bernoulli mixture.} This order fits, by expectation-maximization, a small number of latent query classes, each with its own per-model failure rates, and synthetic failures drawn from the fitted mixture feed the greedy step.

\textbf{Shrinkage to independence and shrinkage to one factor.} Blend the observed failure matrix with synthetic failures from a target model, independence at the observed marginals or a low-rank factor model, at a tuned weight, and apply the greedy step to the blend.

\textbf{Controls.} Model-family lineage is a zero-parameter control that picks the best model from each name-prefix family in round robin, so that the name prefix is a proxy for shared failure structure in a pool built from fine-tunes of a few dozen base models. Accuracy-weighted random samples the order without replacement with design-accuracy weights, and the random control uses a seeded uniform order and no model information.

\textbf{Matched non-family controls.} The exact-independence control orders models by descending router score and stops with the independence product \(\widehat q_x^{\mathrm{ind}}(S)=\prod_{j\in S}(1-\widehat p_j(x))\); it fits no dependence statistic. The fixed design-marginal control orders by decreasing design accuracy \(\bar p_j^D\) with the same stopping rule, which isolates the value of co-failure information in the order. Query-local empirical coverage uses the 40 frozen router-training neighbors to define a local all-miss frequency and greedily minimizes it per query. FlexRouter and BaRP are reconstructed as described in Appendix~\ref{app:implementation}, and their calibration uses the same final-calibration rule as every other policy.

\subsection{Summary of the family}
\label{app:taxonomy}

Twelve alternative orders and the tally default each map the design matrix to one pool ordering, and the hyperparameter grids contain three to five values per estimator.

\begin{center}
\scriptsize
\setlength{\tabcolsep}{4pt}
\begin{tabular}{llp{5.6cm}}
\toprule
family & order & failure matrix fed to the greedy step \\
\midrule
counting & tally (default) & observed design failures \\
counting & double fault & observed failures, pairwise joint-failure objective \\
counting & pairwise count & smoothed pairwise surrogate (ranking heuristic) \\
correlation & error correlation & observed failures, correlation-penalized objective \\
correlation & rank correlation & latent draws, concordance-based correlation \\
latent threshold & Gaussian copula & latent-normal draws, low-rank correlation \\
latent threshold & Student-t copula & heavy-tailed latent draws \\
latent threshold & low-rank factor copula & latent draws at fixed rank \\
mixture & Bernoulli mixture & draws from fitted latent query classes \\
shrinkage & to independence & observed failures blended with independence draws \\
shrinkage & to one factor & observed failures blended with factor-model draws \\
control & lineage, accuracy-weighted random, random & no fitted dependence \\
\bottomrule
\end{tabular}
\end{center}

\section{Protocol details and claim boundaries}
\label{app:protocol}

\textbf{Splits and selection order.} Query-level cells use 20 redraws with design 30 percent, router training 20 percent, final calibration 20 percent, validation 15 percent, and test 15 percent. The model pool is resampled per redraw. Validation selects estimator hyperparameters or a unified candidate using provisional operating points. The provisional cutoff is then discarded, the complete winning map is frozen, and final calibration selects only its threshold, after which the test split is evaluated once.

\textbf{Standard loss and certification.} The primary loss is the standard all-miss loss over real-model sets, under which an empty set has loss one. Calibration uses the finite-sample bound $(c+1)/(n+1)\leq\alpha$. If no threshold passes, the draw is \texttt{NON-CERTIFIED}; its calls and risk are not pooled with certified draws, so every efficiency table reports the certification rate. Structural infeasibility of the complete model pool and finite-calibration non-certification are reported separately.

\textbf{The per-model construction and diagnostics of \citet{racer2026}.} \citet{racer2026} train a base router outside the final-calibration sample; for a query \(x\) it estimates each model's correctness probability \(\widehat p_j(x)\) and scores model \(j\) either by the inverse probability \(1-\widehat p_j(x)\) (RACER-P) or by the gap between its score and that of the highest-ranked model (RACER-G). One calibrated cutoff on these per-model scores yields a query-dependent nested family whose model ranking varies with \(x\) while the cutoff is shared across queries. Substituting the per-model value \(s(x,j)=1-\widehat p_j(x)\) for one minus the cumulative-coverage residual \(r_{k-1}(x)\) in the construction of Section~\ref{sec:setup} gives \(C_\lambda(x)=\{j:1-\widehat p_j(x)\leq\lambda\}\), which is exactly the RACER-P family over the real model pool. RACER-P is thus the special case of our construction whose residual ignores the models already selected, whereas CopRoute lets \(r_{k-1}(x)\) depend on the current prefix so that a candidate's score reflects the coverage already obtained. \citet{racer2026} additionally augment the pool with a virtual null model \(m_\emptyset\), writing \(\cM'=\cM\cup\{m_\emptyset\}\) and placing \(m_\emptyset\) in the augmented ground-truth set whenever no real model is correct, so that selecting only \(m_\emptyset\) is read as abstention. Because this null action always covers the augmented target, their calibration always admits a feasible member above the finite-sample floor, and infeasibility is absorbed into abstention rather than reported. In our comparisons RACER-G and RACER-P use the same splits, the same final-calibration split, and the same standard all-miss loss as the proposed policies. We certify the stricter real-model loss with no always-safe action, so that an empty real-model set incurs loss one, and a draw whose full pool cannot meet the target returns \texttt{NON-CERTIFIED}. The augmented-null loss of \citet{racer2026} is retained only as a separately named diagnostic, and its null action is not interpreted as human review, zero cost, or guaranteed safety.

\textbf{Realized-risk classification.} A certified cell is marked above target only when its mean standard all-miss risk exceeds $\alpha$ by more than two standard errors. This classification summarizes the finite test sample and does not restate the theorem.

\textbf{Synthetic conditional-independence diagnostic.} A separate synthetic experiment generates heterogeneous query difficulty with model correctness independent conditional on that difficulty; the fully specified five-regime successor of this diagnostic is the mechanism map of Appendix~\ref{app:cost}. Each of 128 seeds draws a scalar difficulty \(x\sim\mathcal N(0,1)\) per query and eight models whose conditional correctness probabilities are \(\sigma(\beta_j-\gamma_j x+\kappa_j(x^2-1))\) clipped to \([0.03,0.97]\), with intercepts \(\beta\) from 1.55 down to 0.05, loadings \(\gamma\) between 0.64 and 1.48, and curvatures \(\kappa\) between \(-0.12\) and 0.12; correctness is drawn independently across models given \(x\). Splits of 1{,}500 design, 1{,}000 validation, 1{,}500 final-calibration, and 4{,}000 test queries are used at \(\alpha=0.10\), with a correlation-penalty grid \(\{0.25,0.5,1,2\}\) selected on validation. Pooled off-diagonal failure correlation is 0.152 with 95 percent interval $[0.151,0.153]$, while residual correlation, computed from the residuals \(Y-p(x)\), is 0.0008 with interval $[0.0001,0.0016]$. Exact factorization, pooled correlation, and residual correlation each use 2.0 calls, with risks 0.0656, 0.0657, and 0.0657. Because the three estimators use the same 2.0 calls at risks within 0.0001 of one another while pooled correlation is 0.152, pooled association alone does not indicate exploitable residual dependence.

\section{Detailed calibrated results}
\label{app:calibrated-results}

This section completes the main comparison of Table~\ref{tab:main-calls}. Table~\ref{tab:pools-certified} reports the RouterEval benchmarks at the 200-, 300-, and 500-model pools, including the CopRoute-Dep ablation. Table~\ref{tab:mathint-feasibility} documents the excluded benchmark, MATH-intermediate-algebra, which certifies for no policy at any pool at the main targets and only occasionally at looser targets, so no policy comparison is possible there. The reason is the feasibility condition of Section~\ref{sec:experiments}. Even calling every model in the pool misses 40.4, 34.6, 25.4, and 22.1 percent of the calibration queries at pools 100, 200, 300, and 500, so the condition \((c_M+1)/(n+1)\leq\alpha\) fails on every redraw at the main targets and holds on some redraws only where the pool is large enough and the target loose enough for the full-pool miss rate to fall below the target, at pool 500 for \(\alpha\geq 0.15\) and at pool 300 for \(\alpha=0.20\). At the tightest targets on the hardest pool-100 benchmarks, certification rates also fall for every policy (at $\alpha=0.02$, 3 of 20 redraws on IFEval, 1 on MATH-algebra, none on MuSR); on the redraws that do certify there, dependence-corrected stopping reduces calls on MATH-algebra from 55.1 to 39.2 and on IFEval from 78.4 to 61.4, although with one and three redraws no interval can be formed for these differences.

\begin{table}[t]
\centering
\caption{Certified mean calls on the RouterEval benchmarks at the 200-, 300-, and 500-model pools, completing Table~\ref{tab:main-calls}. A dagger marks cells whose certification rate is below 0.9; a dash means no redraw certifies.}
\label{tab:pools-certified}
\scriptsize
\setlength{\tabcolsep}{3pt}
\begin{tabular}{@{}ll rrrrrr @{\hspace{10pt}} rrrrrr@{}}
\toprule
 & & \multicolumn{6}{c}{$\alpha=0.05$} & \multicolumn{6}{c}{$\alpha=0.10$} \\
\cmidrule(lr){3-8}\cmidrule(lr){9-14}
dataset & pool & Ind & Dep & G & P & Flex. & BaRP & Ind & Dep & G & P & Flex. & BaRP \\
\midrule
BBH-boolean & 200 & \textbf{2.34} & \textbf{2.34} & 4.26 & 4.72 & 4.67 & 2.70 & 1.12 & 1.12 & 2.19 & 3.09 & \textbf{1.11} & 1.28 \\
 & 300 & 2.45 & 2.45 & 4.61 & 5.69 & \textbf{2.11} & 3.03 & 1.27 & 1.27 & 2.16 & 3.33 & \textbf{1.23} & 1.34 \\
 & 500 & \textbf{2.79} & \textbf{2.79} & 5.78 & 5.81 & 2.86 & 4.50 & \textbf{1.40} & \textbf{1.40} & 2.89 & 3.95 & \textbf{1.40} & 1.49 \\
\addlinespace[2pt]
BBH-navigate & 200 & \textbf{5.02} & \textbf{5.02} & 20.7 & 20.1 & 6.34 & 9.14 & \textbf{2.39} & \textbf{2.39} & 4.14 & 4.30 & 4.94 & 5.39 \\
 & 300 & 6.82 & 6.82 & 11.7 & 11.0 & \textbf{5.77} & 8.94 & \textbf{2.09} & \textbf{2.09} & 5.22 & 5.86 & 2.45 & 6.09 \\
 & 500 & \textbf{5.72} & \textbf{5.72} & 16.0 & 16.3 & 11.4 & 7.76 & \textbf{2.00} & \textbf{2.00} & 4.99 & 5.53 & 5.82 & 4.75 \\
\addlinespace[2pt]
GPQA & 200 & \textbf{30.2} & \textbf{30.2} & 38.4 & 38.4 & 43.7 & 41.0 & \textbf{11.1} & \textbf{11.1} & 15.7 & 15.7 & 21.3 & 20.7 \\
 & 300 & \textbf{23.8} & \textbf{23.8} & 34.0 & 34.0 & 35.0 & 33.5 & \textbf{12.8} & \textbf{12.8} & 17.2 & 17.2 & 13.7 & 16.0 \\
 & 500 & 30.6 & 30.6 & 29.4 & 29.4 & \textbf{25.7} & 32.9 & 14.7 & 14.7 & 17.5 & 17.5 & \textbf{12.6} & 15.5 \\
\addlinespace[2pt]
IFEval & 200 & \textbf{13.7} & 20.2 & 17.5 & 20.7 & 17.9 & 15.1 & \textbf{3.87} & 4.87 & 5.86 & 9.67 & 5.38 & 3.95 \\
 & 300 & \textbf{11.8} & 16.3 & 21.5 & 25.6 & 12.1 & 13.3 & \textbf{2.12} & 2.23 & 2.96 & 7.31 & 4.31 & 4.97 \\
 & 500 & 14.6 & 30.3 & 21.0 & 29.1 & \textbf{11.6} & 16.9 & \textbf{2.50} & 3.07 & 4.32 & 9.30 & 4.13 & 4.67 \\
\addlinespace[2pt]
MATH-algebra & 200 & \textbf{27.5} & 33.1 & 28.9 & 28.8 & 55.9 & 58.6 & \textbf{6.98} & 7.15 & 8.17 & 8.50 & 9.81 & 10.0 \\
 & 300 & \textbf{17.7} & 25.4 & 19.8 & 20.1 & 53.3 & 61.1 & \textbf{3.30} & \textbf{3.30} & 3.73 & 3.64 & 7.89 & 6.49 \\
 & 500 & \textbf{22.8} & 23.6 & 24.1 & 23.5 & 55.7 & 26.8 & 4.55 & \textbf{4.53} & 6.48 & 7.10 & 16.1 & 11.2 \\
\addlinespace[2pt]
MuSR & 200 & \textbf{74.7$^\dagger$} & 78.7$^\dagger$ & 124.1$^\dagger$ & 154.2$^\dagger$ & 145.4$^\dagger$ & 145.5$^\dagger$ & \textbf{13.9} & 17.9 & 57.5 & 88.2 & 59.1 & 71.6 \\
 & 300 & \textbf{88.7} & 107.0 & 166.1 & 212.4 & 165.1 & 164.9 & \textbf{7.58} & 7.59 & 60.4 & 113.5 & 60.3 & 73.6 \\
 & 500 & \textbf{30.5} & 48.7 & 190.9 & 251.5 & 132.7 & 135.1 & \textbf{5.80} & 11.4 & 73.3 & 152.7 & 30.0 & 48.4 \\
\addlinespace[2pt]
\bottomrule
\end{tabular}
\end{table}

\begin{figure}[t]
\centering
\includegraphics[width=\linewidth]{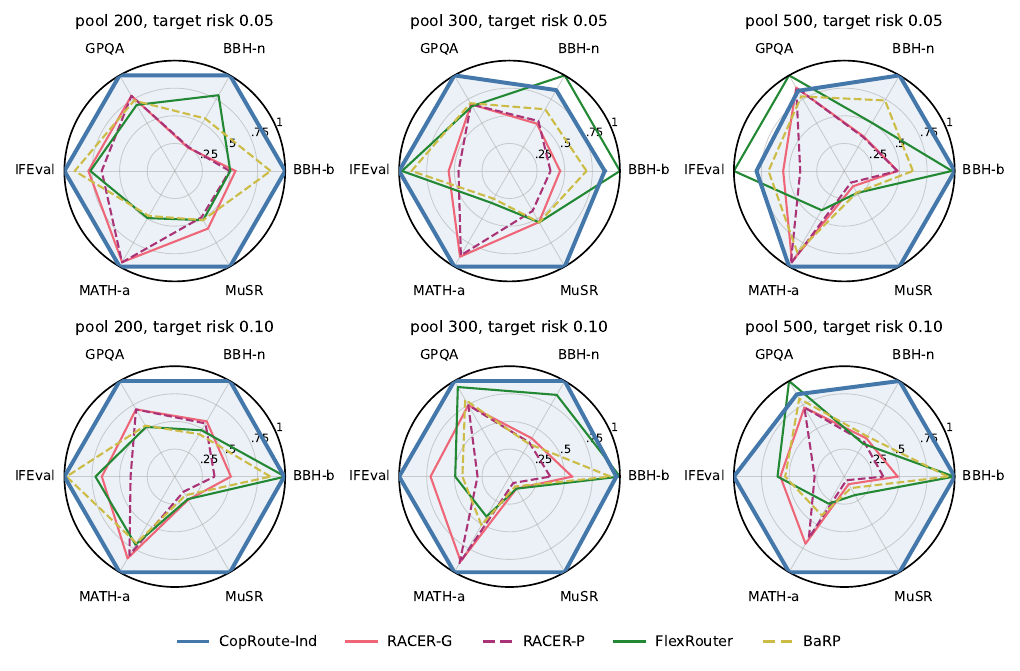}
\caption{Relative call efficiency at the 200-, 300-, and 500-model pools of the six RouterEval datasets, one panel per pool size and target risk. Each value is the best certified calls on the axis divided by the policy's certified calls, so the per-axis winner lies on the rim at 1; FlexRouter and BaRP are averaged over five training seeds. CopRoute is on the rim on 29 of the 36 axes and FlexRouter on the remaining seven. The 500-model panels are repeated in Figure~\ref{fig:radar}.}
\label{fig:radar-pools}
\end{figure}

\begin{table}[t]
\centering
\caption{MATH-intermediate-algebra certification rates for CopRoute-Ind (identical for every policy in each cell). No policy certifies at the targets of Table~\ref{tab:main-calls}, so the dataset is reported here rather than in the main comparison.}
\label{tab:mathint-feasibility}
\scriptsize
\begin{tabular}{@{}r rrrrrrr@{}}
\toprule
pool & $\alpha=.02$ & .03 & .05 & .075 & .10 & .15 & .20 \\
\midrule
100 & 0.00 & 0.00 & 0.00 & 0.00 & 0.00 & 0.00 & 0.00 \\
200 & 0.00 & 0.00 & 0.00 & 0.00 & 0.00 & 0.00 & 0.00 \\
300 & 0.00 & 0.00 & 0.00 & 0.00 & 0.00 & 0.00 & 0.15 \\
500 & 0.00 & 0.00 & 0.00 & 0.00 & 0.00 & 0.05 & 0.30 \\
\bottomrule
\end{tabular}
\end{table}

\section{Complete estimator grids for the CopRoute family}
\label{app:full-results}

Table~\ref{tab:estimator-grid} expands the main-text family comparison (Table~\ref{tab:estimator-family}) to every pool. First, no alternative order improves on the tally order in aggregate under either stopping rule. The gap to the seeded random control (about 7 calls) bounds the value of any fitted order, but it does not isolate the co-failure information, because a random order also discards marginal quality; that contribution is measured by the design-marginal contrast of Table~\ref{tab:dep-ablation}, and on GPQA and BBH-navigate the random order beats tally in some aggregates (Table~\ref{tab:estimator-by-dataset}), so the random gap is not uniformly positive. Second, the alternatives that model dependence with more structure, the Bernoulli mixture and the Student-t copula, beat tally at the 100-model pool and lose at larger pools. With 100 models the design set supports a richer estimate, whereas at 500 models the same design budget yields far fewer observations per prefix and the simple tally count is the more reliable statistic. Third, the two orders built from correlation matrices (error correlation, rank correlation) win individual cells but lose by tens of calls elsewhere, so their aggregate differences against tally are positive, and an order that wins only occasionally cannot replace the tally order by default. The same tables report the dependence-corrected stopping rule for every order.

\begin{table}[t]
\centering
\caption{Ablation of the two components on the cells of Table~\ref{tab:main-calls}. Dep $-$ Ind keeps the tally order and replaces the independence-based stopping rule by the dependence-corrected rule of CopRoute-Dep; marginal $-$ Ind keeps the stopping rule and replaces the tally order by the design-marginal order, which ranks models by design accuracy alone. Entries are paired differences in certified mean calls with bootstrap 95 percent intervals over redraws; bold intervals exclude zero, and a superscript gives the number of jointly certified redraws when it is below 20.}
\label{tab:dep-ablation}
\scriptsize
\setlength{\tabcolsep}{3pt}
\resizebox{\linewidth}{!}{%
\begin{tabular}{@{}l rr rr@{}}
\toprule
 & \multicolumn{2}{c}{$\alpha=0.05$} & \multicolumn{2}{c}{$\alpha=0.10$} \\
\cmidrule(lr){2-3}\cmidrule(lr){4-5}
dataset & Dep $-$ Ind & marginal $-$ Ind & Dep $-$ Ind & marginal $-$ Ind \\
\midrule
ARC & \phantom{+}0.00 (tie) & +0.12 [-0.06, +0.42] & \phantom{+}0.00 (tie) & \phantom{+}0.00 (tie) \\
HellaSwag & \phantom{+}0.00 (tie) & \textbf{+0.61 [+0.49, +0.73]} & \phantom{+}0.00 (tie) & +0.02 [\phantom{+}0.00, +0.04] \\
MMLU & \phantom{+}0.00 (tie) & \textbf{+1.69 [+1.41, +1.93]} & \phantom{+}0.00 (tie) & \textbf{+0.78 [+0.65, +0.89]} \\
WinoGrande & +0.01 [\phantom{+}0.00, +0.02] & \textbf{+2.57 [+2.29, +2.85]} & \phantom{+}0.00 [-0.01, +0.02] & \textbf{+0.95 [+0.70, +1.22]} \\
BBH-boolean & \phantom{+}0.00 (tie) & +0.16 [-0.30, +0.77] & \phantom{+}0.00 (tie) & +0.13 [\phantom{+}0.00, +0.28] \\
BBH-navigate & \phantom{+}0.00 (tie) & \textbf{+21.40 [+13.05, +28.19]} & -0.01 [-0.02, \phantom{+}0.00] & \textbf{+5.32 [+3.50, +7.47]} \\
GPQA & \phantom{+}0.00 (tie) & \textbf{+4.40 [+1.15, +8.05]} & \phantom{+}0.00 (tie) & \textbf{+4.20 [+2.40, +6.05]} \\
IFEval & \textbf{+5.61 [+1.76, +10.19]}$^{19}$ & \textbf{+9.38 [+3.10, +17.00]}$^{19}$ & +0.04 [\phantom{+}0.00, +0.11] & \phantom{+}0.00 [-0.49, +0.47] \\
MATH-algebra & +0.84 [\phantom{+}0.00, +2.52]$^{3}$ & -0.99 [-2.00, \phantom{+}0.00]$^{3}$ & -1.67 [-5.80, +0.69] & -0.39 [-1.09, +0.36] \\
MuSR & \phantom{+}0.00 (tie)$^{8}$ & \textbf{+19.39 [+7.46, +32.41]}$^{8}$ & +1.19 [\phantom{+}0.00, +3.58] & \textbf{+27.25 [+21.31, +33.27]} \\
\bottomrule
\end{tabular}}

\end{table}

\textbf{Dependence-corrected stopping.}
Table~\ref{tab:dep-ablation} isolates the two components of the policy on the cells of Table~\ref{tab:main-calls}. The first contrast keeps the tally order and replaces the independence-based stopping rule by the dependence-corrected rule of CopRoute-Dep. Validation selects the uncorrected rule (\(h=0\)) on most redraws, in 10 to 19 of 20 at every pool on IFEval and MuSR, the two benchmarks for which the selected values were recorded, so the two policies coincide on 12 of the 20 cells and differ by less than 0.05 calls on four more. Where they differ, the correction lowers calls on MATH-algebra at $\alpha=0.10$ by 1.7 with an interval that includes zero, and raises them on IFEval at $\alpha=0.05$ by 5.6 calls and on MuSR at $\alpha=0.10$ by 1.2. Across all 187 certified cells of the four pools, the two rules tie on 104, the correction lowers calls on 21 and raises them on 62, and the six differences whose intervals exclude zero in its favor are at most 0.03 calls or rest on a single redraw, whereas the six against it reach 17 calls (Table~\ref{tab:estimator-grid}). The dependence ratio is estimated from the design queries that miss an entire prefix, which become rare at exactly the depths where the correction would matter, and the resulting noise costs more than the correction recovers, so CopRoute-Ind is the default. The second contrast keeps the stopping rule and replaces the tally order by the design-marginal order. Replacing the tally order costs up to 21 calls on BBH-navigate and 27 on MuSR, and its interval excludes zero against the marginal order in 13 of the 20 cells and in none in its favor, so the gain of CopRoute over the baselines comes from the co-failure order rather than from the stopping correction.

\begin{table}[t]
\centering
\caption{Complete order-by-pool grid of extra certified calls relative to the tally order under the same stopping rule, weighted by certified redraws over the ten datasets and seven targets; negative favors the alternative.}
\label{tab:estimator-grid}
\scriptsize
\setlength{\tabcolsep}{2.6pt}
\begin{tabular}{@{}l rrrrrr@{\hspace{10pt}}rrrrrr@{}}
\toprule
 & \multicolumn{6}{c}{independence-based stopping} & \multicolumn{6}{c}{dependence-corrected stopping} \\
\cmidrule(lr){2-7}\cmidrule(lr){8-13}
order & all & p11 & p100 & p200 & p300 & p500 & all & p11 & p100 & p200 & p300 & p500 \\
\midrule
Bernoulli mixture & +0.42 & +0.10 & -0.65 & -0.32 & +0.67 & +2.01 & +0.36 & +0.08 & -0.39 & +0.49 & +0.47 & +0.94 \\
Student-t copula & +1.77 & +0.42 & -0.68 & -0.35 & +3.94 & +4.63 & +1.35 & +0.39 & -0.67 & -0.48 & +3.25 & +3.55 \\
Gaussian copula & +1.97 & +0.54 & -0.09 & +1.08 & +3.35 & +4.13 & +1.81 & +0.55 & -0.10 & +1.59 & +2.97 & +3.33 \\
rank correlation & +2.04 & +0.54 & +0.26 & +0.71 & +3.10 & +4.76 & +1.83 & +0.54 & +0.20 & +1.31 & +2.96 & +3.44 \\
double fault & +1.59 & +0.03 & +0.13 & +0.38 & +3.22 & +3.40 & +1.13 & +0.02 & +0.00 & +0.04 & +2.79 & +2.21 \\
low-rank factor copula & +2.44 & +0.22 & +0.17 & +0.64 & +3.46 & +6.51 & +2.26 & +0.22 & +0.17 & +1.21 & +3.61 & +5.06 \\
shrinkage to one factor & +2.94 & +0.14 & +0.39 & +1.65 & +4.54 & +6.60 & +2.70 & +0.13 & +0.52 & +1.85 & +4.24 & +5.52 \\
shrinkage to independence & +2.97 & +0.09 & +1.20 & +1.44 & +5.48 & +5.33 & +3.10 & +0.09 & +1.55 & +1.78 & +5.70 & +5.06 \\
error correlation & +6.25 & +0.38 & +0.98 & +7.60 & +5.27 & +14.32 & +5.50 & +0.40 & +1.23 & +6.91 & +4.11 & +12.53 \\
model-family lineage & +4.83 & +0.55 & +5.84 & +6.05 & +2.60 & +7.76 & +4.24 & +0.56 & +5.70 & +6.32 & +1.54 & +6.07 \\
accuracy-weighted random & +4.28 & +1.41 & +2.54 & +2.97 & +4.05 & +9.07 & +4.04 & +1.40 & +2.58 & +2.90 & +3.88 & +8.20 \\
random control & +7.78 & +1.87 & +4.19 & +7.53 & +9.08 & +13.65 & +7.24 & +1.86 & +4.00 & +7.16 & +8.33 & +12.51 \\
\bottomrule
\end{tabular}
\end{table}

Table~\ref{tab:estimator-by-dataset} gives the same comparison per dataset. The alternatives' wins concentrate on GPQA and BBH-navigate, where even the random control improves on tally, and more weakly on MATH-algebra; tally is strongest or tied everywhere else. The error-correlation order shows the heavy-tailed pattern behind its aggregate loss, gaining 20 calls on MuSR and 11 on GPQA but losing 97 on MATH-algebra.

\begin{table}[t]
\centering
\caption{Dataset-wise extra certified calls relative to the tally order under the same stopping rule, weighted by certified redraws over each dataset's pools and the seven targets; negative (bold) favors the alternative. RouterBench columns use the 11-model pool and RouterEval columns aggregate pools 100 to 500.}
\label{tab:estimator-by-dataset}
\scriptsize
\setlength{\tabcolsep}{2.6pt}
\begin{tabular}{@{}l rrrrrrrrrr@{}}
\toprule
order & ARC & HSwag & MMLU & WinoG & BBH-b & BBH-n & GPQA & IFEval & MATH-a & MuSR \\
\midrule
\multicolumn{11}{@{}l}{\emph{independence-based stopping}} \\
Bernoulli mixture & +0.16 & +0.15 & +0.09 & +0.00 & \textbf{-0.04} & \textbf{-2.51} & \textbf{-0.40} & +4.70 & \textbf{-0.16} & +1.72 \\
Student-t copula & +0.11 & +0.32 & +0.37 & +0.82 & \textbf{-0.46} & \textbf{-1.78} & \textbf{-1.26} & +4.94 & \textbf{-0.21} & +12.65 \\
Gaussian copula & +0.14 & +0.34 & +0.47 & +1.18 & \textbf{-0.03} & \textbf{-1.17} & \textbf{-0.79} & +5.26 & \textbf{-0.28} & +11.83 \\
rank correlation & +0.15 & +0.33 & +0.45 & +1.16 & \textbf{-0.03} & \textbf{-1.59} & \textbf{-1.09} & +5.52 & \textbf{-0.23} & +13.01 \\
double fault & +0.08 & +0.05 & \textbf{-0.01} & +0.01 & +1.06 & +3.67 & \textbf{-7.31} & +5.71 & +2.61 & +5.57 \\
low-rank factor copula & +0.17 & +0.21 & +0.42 & +0.12 & +0.21 & \textbf{-1.83} & +0.50 & +5.43 & \textbf{-0.27} & +14.96 \\
shrinkage to one factor & +0.16 & +0.15 & +0.11 & +0.15 & +0.21 & +0.20 & +1.52 & +4.54 & \textbf{-0.10} & +16.19 \\
shrinkage to independence & +0.08 & +0.14 & +0.10 & +0.04 & +0.18 & +0.07 & +1.64 & +5.15 & \textbf{-0.26} & +16.00 \\
error correlation & \textbf{-0.06} & +0.83 & +0.93 & +0.02 & \textbf{-1.02} & \textbf{-4.34} & \textbf{-10.89} & \textbf{-5.12} & +97.05 & \textbf{-20.40} \\
model-family lineage & \textbf{-0.06} & +0.41 & +0.79 & +1.11 & \textbf{-0.15} & +9.51 & +0.47 & \textbf{-0.25} & +2.04 & +24.01 \\
accuracy-weighted random & +0.86 & +1.74 & +1.24 & +1.76 & +0.44 & \textbf{-1.33} & \textbf{-4.84} & \textbf{-3.22} & +11.67 & +31.97 \\
random control & +1.40 & +2.19 & +2.13 & +1.85 & +0.18 & \textbf{-2.18} & \textbf{-6.72} & +0.77 & +43.48 & +27.21 \\
\midrule
\multicolumn{11}{@{}l}{\emph{dependence-corrected stopping}} \\
Bernoulli mixture & +0.10 & +0.14 & +0.09 & +0.00 & \textbf{-0.04} & \textbf{-2.51} & \textbf{-0.40} & +5.90 & +0.31 & \textbf{-0.71} \\
Student-t copula & +0.06 & +0.33 & +0.38 & +0.77 & \textbf{-0.46} & \textbf{-1.73} & \textbf{-1.26} & +4.58 & \textbf{-0.60} & +9.96 \\
Gaussian copula & +0.16 & +0.34 & +0.47 & +1.16 & \textbf{-0.02} & \textbf{-1.09} & \textbf{-0.79} & +6.55 & \textbf{-0.70} & +9.29 \\
rank correlation & +0.15 & +0.33 & +0.45 & +1.16 & \textbf{-0.04} & \textbf{-1.38} & \textbf{-1.09} & +6.88 & \textbf{-0.80} & +9.92 \\
double fault & +0.05 & +0.04 & \textbf{-0.01} & +0.01 & +1.06 & +3.46 & \textbf{-7.31} & +4.41 & +2.61 & +3.66 \\
low-rank factor copula & +0.19 & +0.20 & +0.42 & +0.12 & +0.22 & \textbf{-1.67} & +0.50 & +7.03 & \textbf{-0.79} & +11.93 \\
shrinkage to one factor & +0.14 & +0.14 & +0.11 & +0.14 & +0.20 & +0.30 & +1.52 & +4.19 & \textbf{-0.11} & +14.53 \\
shrinkage to independence & +0.09 & +0.13 & +0.10 & +0.04 & +0.18 & +0.27 & +1.64 & +7.91 & \textbf{-0.42} & +13.67 \\
error correlation & \textbf{-0.03} & +0.86 & +0.97 & +0.02 & \textbf{-1.02} & \textbf{-4.34} & \textbf{-10.89} & \textbf{-6.33} & +95.63 & \textbf{-23.72} \\
model-family lineage & \textbf{-0.08} & +0.43 & +0.80 & +1.13 & \textbf{-0.13} & +9.98 & +0.47 & \textbf{-2.74} & +1.54 & +22.03 \\
accuracy-weighted random & +0.84 & +1.75 & +1.24 & +1.74 & +0.50 & \textbf{-1.44} & \textbf{-4.84} & \textbf{-5.65} & +12.62 & +32.05 \\
random control & +1.41 & +2.17 & +2.14 & +1.81 & +0.19 & \textbf{-2.27} & \textbf{-6.72} & \textbf{-1.29} & +43.55 & +25.30 \\
\bottomrule
\end{tabular}
\end{table}

\section{Transfer across datasets for all policies}
\label{app:transfer}

Figure~\ref{fig:ood-all} repeats the main-text transfer analysis (Figure~\ref{fig:ood-premium}) for every policy on RouterEval, and Table~\ref{tab:id-ood-means} summarizes both suites. Source and target must share a model pool, so transfers run within RouterBench and within RouterEval only. Within RouterEval we exclude source-target pairs from the same benchmark family, which affects only BBH-boolean and BBH-navigate. An earlier variant that excluded only the target itself made transfer into a BBH task look successful because the source pool contained its sibling, which is near in-domain rather than out-of-domain transfer. The learned baselines used in this suite were likewise retrained with the target's whole family held out. The seventh RouterEval benchmark, MATH-intermediate-algebra, is absent because it certifies for no policy (Table~\ref{tab:mathint-feasibility}). The premiums differ by construction, since CopRoute fits the most target-specific structure and so loses the most when that structure comes from the wrong dataset, while FlexRouter and BaRP change little because their calls are set mainly by their learned routers. The levels still favor CopRoute, because averaged over every directed out-of-domain pair, CopRoute-Ind has the lowest mean calls on RouterBench at both targets and on RouterEval at $\alpha=0.05$, and sits within 1.1 calls of FlexRouter at RouterEval $\alpha=0.10$. Table~\ref{tab:transfer-summary} reports realized risk, where reusing a source-calibrated cutoff inflates it for every policy, on RouterEval by 0.12 to 0.22, so recalibration on the deployment distribution is required for every policy.

\begin{table}[t]
\centering
\caption{Mean certified calls in domain (ID) and out of domain (OOD, averaged over every directed source-target pair in the suite), recalibrated on the target.}
\label{tab:id-ood-means}
\scriptsize
\begin{tabular}{@{}l rr rr @{\hspace{12pt}} rr rr@{}}
\toprule
 & \multicolumn{4}{c}{RouterBench (pool 11)} & \multicolumn{4}{c}{RouterEval (pool 100)} \\
\cmidrule(lr){2-5}\cmidrule(lr){6-9}
 & \multicolumn{2}{c}{$\alpha=0.05$} & \multicolumn{2}{c}{$\alpha=0.10$} & \multicolumn{2}{c}{$\alpha=0.05$} & \multicolumn{2}{c}{$\alpha=0.10$} \\
policy & ID & OOD & ID & OOD & ID & OOD & ID & OOD \\
\midrule
CopRoute-Ind & 3.65 & 5.29 & 1.90 & 2.69 & 23.44 & 33.26 & 6.60 & 17.23 \\
CopRoute-Dep & 3.65 & 5.30 & 1.90 & 2.69 & 24.52 & 35.28 & 6.53 & 17.96 \\
RACER-G & 4.81 & 5.44 & 2.89 & 2.91 & 31.91 & 34.81 & 13.73 & 18.50 \\
RACER-P & 5.11 & 5.86 & 3.35 & 3.36 & 32.20 & 34.62 & 15.59 & 18.46 \\
FlexRouter & 5.05 & 5.32 & 2.45 & 2.74 & 35.70 & 37.59 & 15.03 & 16.08 \\
BaRP & 5.22 & 5.52 & 2.57 & 2.89 & 33.37 & 35.07 & 16.76 & 17.87 \\
\bottomrule
\end{tabular}
\end{table}

\begin{table}[t]
\centering
\caption{Cross-dataset transfer summary of the mean extra certified calls after out-of-domain fitting with target recalibration (fit premium), and the mean excess of realized risk over the target when a source-calibrated cutoff is reused (risk inflation).}
\label{tab:transfer-summary}
\scriptsize
\begin{tabular}{@{}l rr rr@{}}
\toprule
 & \multicolumn{2}{c}{fit call premium} & \multicolumn{2}{c}{source-calibrated risk inflation} \\
\cmidrule(lr){2-3}\cmidrule(lr){4-5}
policy & RouterBench & RouterEval & RouterBench & RouterEval \\
\midrule
CopRoute-Ind & +0.89 & +16.53 & +0.010 & +0.222 \\
CopRoute-Dep & +0.90 & +18.63 & +0.010 & +0.222 \\
RACER-G & +0.13 & +5.38 & +0.016 & +0.169 \\
RACER-P & +0.06 & -2.04 & +0.015 & +0.165 \\
FlexRouter & +0.19 & +0.09 & +0.002 & +0.140 \\
BaRP & +0.18 & +0.05 & +0.002 & +0.121 \\
\bottomrule
\end{tabular}
\end{table}

\begin{figure}[t]
\centering
\includegraphics[width=0.95\linewidth]{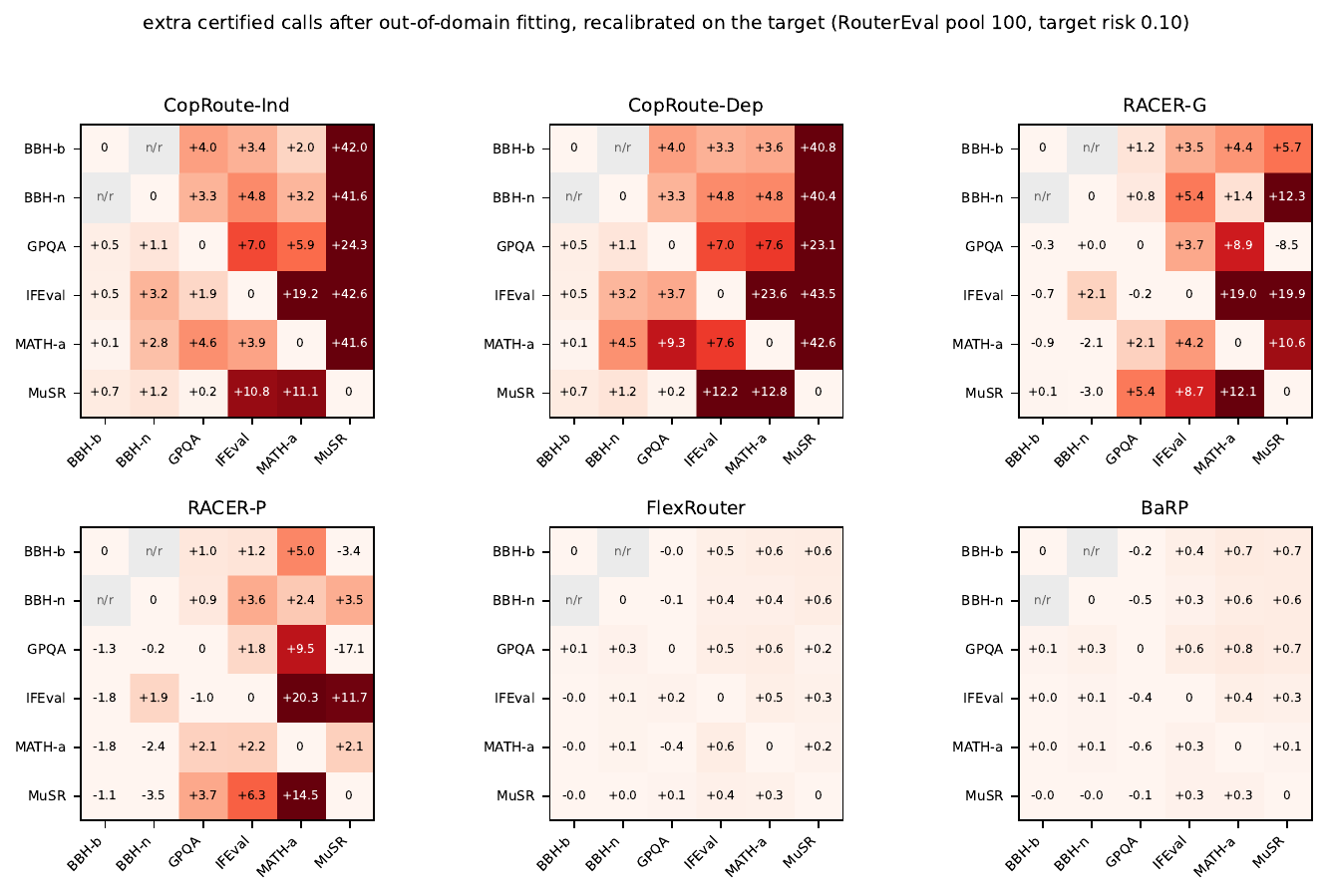}
\caption{Extra certified calls after out-of-domain fitting for every policy, recalibrated on the target, RouterEval pool 100 at target risk 0.10. Cell values are exact; the color scale is shared and capped at 12. The two pairs within the BBH family were deliberately not run (n/r), so that no transfer is near in-domain.}
\label{fig:ood-all}
\end{figure}

\section{Supplementary native learned-baseline metrics}
\label{app:learned-native}

The common comparison evaluates the learned baselines under the same calibrated standard-loss rule as every other policy, and this section reports each source method's native metrics only as supplementary characterization.

\begin{table}[t]
\centering
\caption{Native FlexRouter adaptive stopping on IFEVAL. Entries are mean calls / Success, averaged over five training seeds and 40 redraws per seed. No risk certificate is claimed for this native rule.}
\label{tab:flex-native-ifeval}
\begin{tabular}{rcccc}
\toprule
$\tau$ & pool 100 & pool 200 & pool 300 & pool 500 \\
\midrule
0.0 & 10.00 / 0.934 & 10.00 / 0.935 & 10.00 / 0.952 & 10.00 / 0.960 \\
0.01 & 10.00 / 0.934 & 10.00 / 0.935 & 10.00 / 0.952 & 10.00 / 0.960 \\
0.05 & 10.00 / 0.934 & 10.00 / 0.935 & 10.00 / 0.952 & 10.00 / 0.960 \\
0.1 & 10.00 / 0.934 & 10.00 / 0.935 & 10.00 / 0.952 & 10.00 / 0.960 \\
0.2 & 10.00 / 0.934 & 10.00 / 0.935 & 10.00 / 0.952 & 9.88 / 0.958 \\
0.3 & 9.99 / 0.934 & 9.99 / 0.935 & 9.76 / 0.951 & 8.70 / 0.947 \\
0.5 & 5.83 / 0.882 & 6.43 / 0.910 & 4.41 / 0.905 & 3.09 / 0.872 \\
\bottomrule
\end{tabular}
\end{table}
\begin{table}[t]
\centering
\caption{Native FlexRouter adaptive stopping on MUSR. Entries are mean calls / Success, averaged over five training seeds and 40 redraws per seed. No risk certificate is claimed for this native rule.}
\label{tab:flex-native-musr}
\begin{tabular}{rcccc}
\toprule
$\tau$ & pool 100 & pool 200 & pool 300 & pool 500 \\
\midrule
0.0 & 10.00 / 0.796 & 10.00 / 0.767 & 10.00 / 0.783 & 10.00 / 0.859 \\
0.01 & 10.00 / 0.796 & 10.00 / 0.767 & 10.00 / 0.783 & 10.00 / 0.859 \\
0.05 & 10.00 / 0.796 & 10.00 / 0.767 & 10.00 / 0.783 & 10.00 / 0.859 \\
0.1 & 10.00 / 0.796 & 10.00 / 0.767 & 10.00 / 0.783 & 10.00 / 0.859 \\
0.2 & 10.00 / 0.796 & 10.00 / 0.767 & 10.00 / 0.783 & 9.80 / 0.857 \\
0.3 & 9.99 / 0.796 & 10.00 / 0.767 & 9.66 / 0.782 & 8.90 / 0.851 \\
0.5 & 5.31 / 0.721 & 5.63 / 0.696 & 4.30 / 0.684 & 3.89 / 0.767 \\
\bottomrule
\end{tabular}
\end{table}

\begin{figure}[t]
\centering
\includegraphics[width=0.90\linewidth]{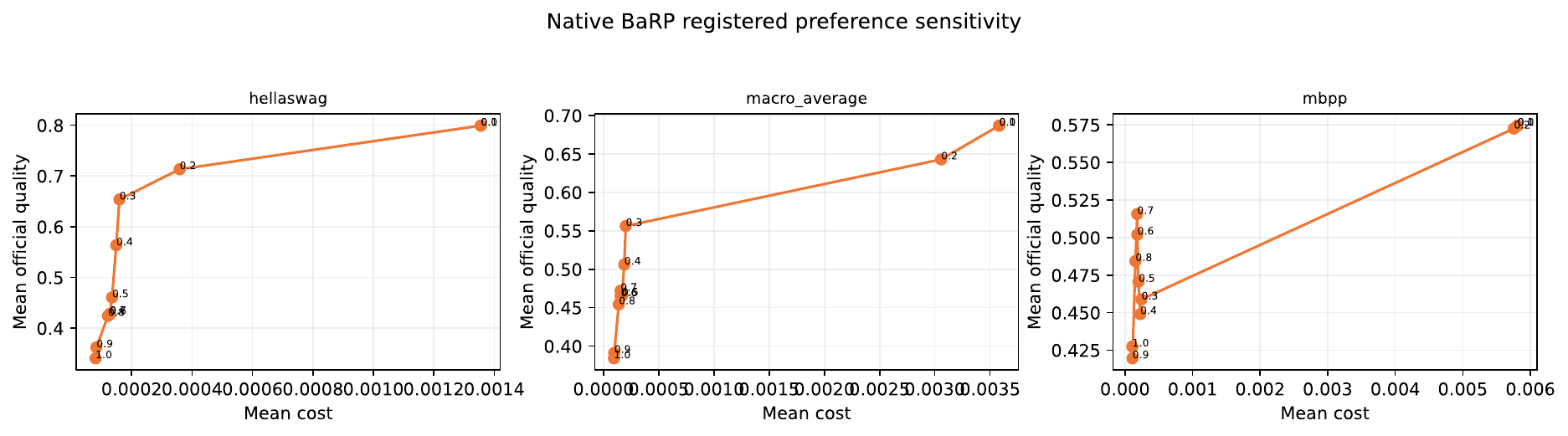}
\caption{Native BaRP quality-cost sensitivity after training on the three nondegenerate RouterBench task families. These single-model source metrics are supplementary to the common set-routing evaluation.}
\label{fig:barp-native}
\end{figure}

\begin{table}[t]
\centering
\caption{Native BaRP preference sensitivity. Entries are official quality / mean cost multiplied by $10^3$. The macro column averages HellaSwag and MBPP. Values aggregate five training seeds.}
\label{tab:barp-native}
\begin{tabular}{ccccc}
\toprule
quality wt. & cost wt. & HellaSwag & MBPP & macro \\
\midrule
1.0 & 0.0 & 0.799 / 1.356 & 0.575 / 5.810 & 0.687 / 3.583 \\
0.9 & 0.1 & 0.799 / 1.356 & 0.575 / 5.810 & 0.687 / 3.583 \\
0.8 & 0.2 & 0.713 / 0.359 & 0.573 / 5.755 & 0.643 / 3.057 \\
0.7 & 0.3 & 0.654 / 0.159 & 0.459 / 0.243 & 0.556 / 0.201 \\
0.6 & 0.4 & 0.563 / 0.149 & 0.449 / 0.227 & 0.506 / 0.188 \\
0.5 & 0.5 & 0.461 / 0.135 & 0.471 / 0.200 & 0.466 / 0.167 \\
0.4 & 0.6 & 0.428 / 0.128 & 0.502 / 0.183 & 0.465 / 0.155 \\
0.3 & 0.7 & 0.428 / 0.127 & 0.516 / 0.181 & 0.472 / 0.154 \\
0.2 & 0.8 & 0.424 / 0.122 & 0.484 / 0.155 & 0.454 / 0.138 \\
0.1 & 0.9 & 0.362 / 0.083 & 0.420 / 0.111 & 0.391 / 0.097 \\
0.0 & 1.0 & 0.341 / 0.081 & 0.427 / 0.109 & 0.384 / 0.095 \\
\bottomrule
\end{tabular}
\end{table}

We omit GSM8K from the native BaRP training tasks because BaRP conditions on prompt features, and in the RouterBench release the GSM8K prompt field contains at most one distinct prompt, which makes the prompt-conditioned formulation trivial on that task. We therefore train the native run on MMLU, ARC-Challenge, and Winogrande, substituting no other task, and keep HellaSwag and MBPP as held-out target tasks.

\section{Cost, answer selection, and further analyses}
\label{app:cost}

\textbf{Cost per query and end-to-end accuracy.}
The certified guarantee bounds the probability that every called model is wrong; a deployed system must additionally return one answer and pay a price per call. Table~\ref{tab:cost-case-study} reports both on RouterBench MMLU, with the answer taken from the called model with the highest router score and each call priced by the public per-model price sheet. Fewer certified calls lower the realized cost. At $\alpha=0.10$, CopRoute-Ind answers 80.97 percent of queries at USD 1.90 per thousand, against 80.18 percent at USD 3.55 for RACER-G, a 46.6 percent reduction at the same certified risk. Final-answer accuracy depends on the answer selector and the called set, and the answer selector adds 2.3 microseconds per query. The strongest single model is cheaper than any calibrated policy but misses 18.6 percent of queries, far above both targets, so no certificate is attainable without routing.

\begin{table}[t]
\centering
\caption{Cost and final-answer accuracy on RouterBench MMLU at matched certified risk; entries are means over 20 redraws with standard errors in brackets, and the \colorbox{green!18}{best} calibrated value per column is shaded. The last two rows are fixed references reported without calibration.}
\label{tab:cost-case-study}
\scriptsize
\setlength{\tabcolsep}{3.2pt}
\begin{tabular}{@{}l rrr @{\hspace{12pt}} rrr@{}}
\toprule
 & \multicolumn{3}{c}{$\alpha=0.05$} & \multicolumn{3}{c}{$\alpha=0.10$} \\
\cmidrule(lr){2-4}\cmidrule(lr){5-7}
policy & calls & accuracy (\%) & cost (USD/1k) & calls & accuracy (\%) & cost (USD/1k) \\
\midrule
CopRoute-Ind & \cellcolor{green!18}7.75 (0.27) & \cellcolor{green!18}80.44 (0.19) & \cellcolor{green!18}3.41 (0.10) & \cellcolor{green!18}2.99 (0.05) & \cellcolor{green!18}80.97 (0.18) & \cellcolor{green!18}1.90 (0.05) \\
CopRoute-Dep & \cellcolor{green!18}7.75 (0.27) & \cellcolor{green!18}80.44 (0.19) & \cellcolor{green!18}3.41 (0.10) & \cellcolor{green!18}2.99 (0.05) & \cellcolor{green!18}80.97 (0.18) & \cellcolor{green!18}1.90 (0.05) \\
RACER-G & 8.66 (0.10) & 80.19 (0.21) & 4.24 (0.01) & 5.64 (0.07) & 80.18 (0.21) & 3.55 (0.03) \\
RACER-P & 8.64 (0.04) & 80.19 (0.21) & 4.22 (0.01) & 6.45 (0.04) & 80.13 (0.22) & 3.66 (0.02) \\
exact independence & 9.45 (0.18) & 80.19 (0.21) & 4.28 (0.03) & 3.87 (0.07) & 80.19 (0.21) & 2.83 (0.03) \\
design-marginal order & 9.44 (0.18) & 80.20 (0.21) & 4.30 (0.03) & 3.76 (0.08) & 80.66 (0.19) & 2.86 (0.06) \\
\midrule
call every model & 11.00 (0.00) & 80.19 (0.21) & 4.47 (0.01) & 11.00 (0.00) & 80.19 (0.21) & 4.47 (0.01) \\
best single model (1 call) & 1.00 (0.00) & 81.39 (0.19) & 1.43 (0.00) & 1.00 (0.00) & 81.39 (0.19) & 1.43 (0.00) \\
\bottomrule
\end{tabular}
\end{table}

\textbf{Cost-weighted stopping.}
Model prices on RouterBench differ by two orders of magnitude, so the cheapest certified set is not always the smallest. The cost-weighted variant of Table~\ref{tab:cost-weighted} replaces the tally objective by newly covered design failures per unit of frozen mean price, so that it fits a different order (identical to the tally order only when all prices are equal), and then applies the same independence-based stopping rule and calibration; it is therefore a cost-aware member of the family whose certificate follows from Theorem~\ref{thm:frozen-feasible-crc} and whose ordering guarantee is the price-weighted analogue of Theorem~\ref{thm:mssc-tight} rather than that theorem itself. It lowers realized cost in every cell, with all bootstrap intervals excluding zero, at a price of 0.4 to 2.7 extra calls; at $\alpha=0.10$ its certified cost is 2 to 11 times lower than count-optimized CopRoute-Ind. Under heterogeneous prices the deployment recommendation is therefore to calibrate the cost-weighted objective.

\begin{table}[t]
\centering
\caption{Cost-weighted stopping on RouterBench, reporting the certified mean cost per 1{,}000 queries in USD, and paired differences against count-optimized CopRoute-Ind with bootstrap 95 percent intervals; every cost interval excludes zero.}
\label{tab:cost-weighted}
\scriptsize
\setlength{\tabcolsep}{2.8pt}
\begin{tabular}{@{}ll rrr @{\hspace{12pt}} rr@{}}
\toprule
 & & \multicolumn{3}{c}{cost, USD per 1k queries} & \multicolumn{2}{c}{cost-weighted minus CopRoute-Ind} \\
\cmidrule(lr){3-5}\cmidrule(lr){6-7}
dataset & $\alpha$ & cost-weighted & CopRoute-Ind & RACER-G & $\Delta$cost [CI95] & $\Delta$calls [CI95] \\
\midrule
ARC & 0.05 & \cellcolor{green!18}0.183 & 0.996 & 1.049 & -0.813 [-0.891, -0.756] & +2.73 [+2.53, +2.92] \\
 & 0.10 & \cellcolor{green!18}0.104 & 0.924 & 0.893 & -0.820 [-0.829, -0.812] & +1.82 [+1.62, +2.00] \\
\addlinespace[2pt]
HellaSwag & 0.05 & \cellcolor{green!18}1.089 & 2.623 & 3.684 & -1.533 [-1.663, -1.409] & +2.66 [+2.46, +2.87] \\
 & 0.10 & \cellcolor{green!18}0.431 & 2.277 & 2.176 & -1.845 [-1.861, -1.830] & +1.44 [+1.35, +1.53] \\
\addlinespace[2pt]
MMLU & 0.05 & \cellcolor{green!18}3.046 & 3.405 & 4.238 & -0.359 [-0.469, -0.247] & +0.71 [+0.45, +0.96] \\
 & 0.10 & \cellcolor{green!18}0.680 & 1.895 & 3.552 & -1.216 [-1.310, -1.136] & +1.87 [+1.79, +1.94] \\
\addlinespace[2pt]
WinoGrande & 0.05 & \cellcolor{green!18}0.062 & 0.567 & 1.317 & -0.505 [-0.508, -0.502] & +0.40 [+0.34, +0.47] \\
 & 0.10 & \cellcolor{green!18}0.050 & 0.547 & 1.004 & -0.496 [-0.498, -0.494] & +0.74 [+0.67, +0.82] \\
\addlinespace[2pt]
\bottomrule
\end{tabular}
\end{table}

\textbf{Cascade routing without response signals.}
Cascade routing decides whether to escalate after reading each model's response, using quality signals computed from that response. Our benchmarks record only per-query correctness labels; the responses, and therefore the quality signals, are unavailable, and using the correctness label in their place would give the cascade access to the ground truth it is meant to predict. We therefore evaluate the cascade with its escalation decision restricted to pre-response information, in which form it reduces to a router. Table~\ref{tab:cascade} reports the consequence on the same splits as CopRoute-Ind, where the restricted cascade exceeds the risk target in seven of eight cells, by up to 3.5 times the target on MATH-algebra at $\alpha=0.10$ (0.350 against a matched 0.110), and provides no distribution-free certificate. Cascades optimize an expected cost-quality tradeoff, while CopRoute certifies a miss rate at minimum calls, so the two frameworks answer different questions.

\begin{table}[t]
\centering
\caption{Realized all-miss risk of the pre-response cascade against matched CopRoute-Ind on identical splits; a star marks values above the target. The cascade carries no distribution-free certificate in any cell.}
\label{tab:cascade}
\scriptsize
\begin{tabular}{@{}l rr @{\hspace{12pt}} rr@{}}
\toprule
 & \multicolumn{2}{c}{$\alpha=0.05$} & \multicolumn{2}{c}{$\alpha=0.10$} \\
\cmidrule(lr){2-3}\cmidrule(lr){4-5}
dataset & cascade & CopRoute-Ind & cascade & CopRoute-Ind \\
\midrule
ARC (pool 11) & 0.093$^*$ & 0.032 & 0.099 & 0.033 \\
GPQA (pool 100) & 0.076$^*$ & 0.066$^*$ & 0.139$^*$ & 0.121$^*$ \\
IFEval (pool 100) & 0.100$^*$ & 0.049 & 0.315$^*$ & 0.095 \\
MATH-algebra (pool 100) & 0.139$^*$ & 0.065$^*$ & 0.350$^*$ & 0.110$^*$ \\
\bottomrule
\end{tabular}
\end{table}

\textbf{Distance of the tally order from the best fixed order.}
Exhaustive dynamic programming over all $2^{11}$ subsets of the 11-model pool, on all 80 RouterBench redraws, gives the exact optimum of the design objective. The calibrated difference, exact minus tally, is $-0.019$ calls with bootstrap 95 percent interval $[-0.033,-0.004]$, with 5 cell-target wins, 16 ties, and 5 losses; a beam-search relaxation ($[-0.118,+0.017]$) and a local-swap refinement ($[-0.015,+0.035]$) both include zero and select the plain tally order in the large majority of redraws. The greedy order is within 0.02 calls of the exact optimum, which is consistent with the factor-four guarantee of Theorem~\ref{thm:mssc-tight} being loose on these instances.

\textbf{Pool composition shift.}
Table~\ref{tab:pool-shift} complements the dataset transfer of Appendix~\ref{app:transfer} by shifting the model pool itself on every cell of the main comparison. Removing one model from the pool invalidates frozen certificates almost entirely, since only 3.4 percent of them (25 of 729) still certify on the reduced pool, whereas refitting the order and recalibrating the cutoff on the reduced pool certifies 63.1 percent of redraws at a mean premium of 1.07 calls. Adding one model leaves every frozen certificate valid, and refitting changes the certified calls by 0.07 on average. The operational rule matches the transfer analysis, in that the fitted structure can be refit at a mean premium of about one call whereas the cutoff must be recalibrated whenever the deployment pool or distribution changes.

\begin{table}[t]
\centering
\caption{Pool composition shift over 51 dataset, pool, and target cells of RouterBench and RouterEval (1{,}020 redraws, of which 729 certify on the unshifted pool, a rate of 71.5 percent). A frozen certificate survives when the unshifted order and cutoff still certify on the shifted pool without any refitting. The last two columns refit the order and recalibrate the cutoff on the shifted pool and report the fraction of all 1{,}020 redraws that then certify and the mean change in certified calls relative to the unshifted pool.}
\label{tab:pool-shift}
\scriptsize
\begin{tabular}{@{}ll rrr@{}}
\toprule
policy & intervention & frozen certificates surviving & certified after refit & mean call change \\
\midrule
CopRoute-Ind & remove one model & 3.4\% (25 of 729) & 63.1\% & +1.07 \\
CopRoute-Ind & add one model & 100.0\% (729 of 729) & 73.6\% & +0.07 \\
CopRoute-Dep & remove one model & 2.2\% (16 of 729) & 63.1\% & +0.98 \\
CopRoute-Dep & add one model & 100.0\% (729 of 729) & 73.6\% & -0.02 \\
\bottomrule
\end{tabular}
\end{table}

\textbf{Sensitivity to router degradation.}
The router-swap result in Section~\ref{sec:experiments} holds between two routers of comparable accuracy. Table~\ref{tab:router-degradation} degrades the router directly, by temperature scaling, by systematic over- and underconfidence, and by adding Gaussian noise with standard deviation up to 0.40 to its scores, and refits nothing else. On RouterBench MMLU at $\alpha=0.10$, CopRoute-Ind's certified calls move by at most $+0.38$ under any transformation (systematic overconfidence) and by $+0.05$ under the strongest noise, and on IFEval by at most $+0.03$, whereas under the strongest noise on IFEval exact independence moves by $+0.59$ and RACER-G by $+1.13$ with an interval of $[+0.38,+1.84]$. The realized risk stays within sampling error of the target in every cell, since calibration absorbs score quality into calls rather than risk. Together with the RouterDC swap, these transformations cover the router-sensitivity range from degraded to strong, and the co-failure ordering is the component least affected by the router.

\begin{table}[t]
\centering
\caption{Router degradation at $\alpha=0.10$, measured as the change in certified mean calls relative to the unmodified KNN router when its scores are transformed and nothing else is refitted, paired over 20 redraws with bootstrap 95 percent intervals; bold intervals exclude zero. The realized test all-miss risk stays within sampling error of the target in every cell (0.085 to 0.103 across the three policies), so degradation is paid in calls rather than in risk.}
\label{tab:router-degradation}
\scriptsize
\setlength{\tabcolsep}{4pt}
\begin{tabular}{@{}l rrr@{}}
\toprule
score transformation & CopRoute-Ind & exact independence & RACER-G \\
\midrule
\multicolumn{4}{@{}l}{\emph{RouterBench MMLU, 11 models}} \\
temperature scaling & -0.01 [-0.03, +0.01] & \textbf{+0.03 [+0.01, +0.05]} & \textbf{-0.11 [-0.14, -0.08]} \\
systematic overconfidence & \textbf{+0.38 [+0.32, +0.44]} & \textbf{+0.15 [+0.08, +0.21]} & \textbf{+0.46 [+0.41, +0.52]} \\
systematic underconfidence & \textbf{+0.05 [+0.02, +0.09]} & \textbf{+0.17 [+0.10, +0.23]} & \textbf{-0.51 [-0.57, -0.45]} \\
Gaussian noise, $\sigma=0.05$ & -0.01 [-0.02, +0.01] & -0.01 [-0.05, +0.04] & +0.01 [-0.02, +0.05] \\
Gaussian noise, $\sigma=0.10$ & \phantom{+}0.00 [-0.01, +0.02] & -0.03 [-0.08, +0.02] & -0.01 [-0.04, +0.02] \\
Gaussian noise, $\sigma=0.20$ & +0.01 [-0.01, +0.03] & -0.05 [-0.11, \phantom{+}0.00] & \textbf{-0.11 [-0.15, -0.06]} \\
Gaussian noise, $\sigma=0.40$ & \textbf{+0.04 [+0.02, +0.07]} & +0.05 [-0.02, +0.13] & \textbf{-0.23 [-0.30, -0.16]} \\
\midrule
\multicolumn{4}{@{}l}{\emph{RouterEval IFEval, 100 models}} \\
temperature scaling & -0.01 [-0.04, +0.03] & \phantom{+}0.00 [-0.01, +0.02] & \textbf{+0.23 [+0.10, +0.36]} \\
systematic overconfidence & \phantom{+}0.00 [-0.13, +0.12] & +0.01 [-0.02, +0.05] & -0.15 [-0.50, +0.11] \\
systematic underconfidence & +0.02 [-0.06, +0.12] & -0.01 [-0.08, +0.04] & \textbf{+0.17 [+0.05, +0.29]} \\
Gaussian noise, $\sigma=0.05$ & -0.05 [-0.10, \phantom{+}0.00] & -0.06 [-0.23, +0.10] & -0.07 [-0.43, +0.22] \\
Gaussian noise, $\sigma=0.10$ & -0.06 [-0.13, +0.02] & -0.10 [-0.45, +0.18] & +0.01 [-0.44, +0.40] \\
Gaussian noise, $\sigma=0.20$ & -0.04 [-0.15, +0.08] & -0.02 [-0.62, +0.48] & -0.02 [-0.75, +0.62] \\
Gaussian noise, $\sigma=0.40$ & +0.03 [-0.14, +0.24] & +0.59 [\phantom{+}0.00, +1.14] & \textbf{+1.13 [+0.38, +1.84]} \\
\bottomrule
\end{tabular}
\end{table}

\textbf{Synthetic mechanism map.}
Table~\ref{tab:mechanism-map} tests the estimators on five synthetic regimes with known ground-truth dependence. The two regimes without residual dependence are negative controls in which no dependence-aware ordering wins a single cell, including the regime where pooled query difficulty creates large raw correlation with no exploitable structure. Where residual dependence is present, the ranking follows its form, so that with redundant model families the pairwise count wins 14 of 15 cells, with complementary specialists the query-local ordering wins from the smallest budget, and in the higher-order regime, whose dependence is invisible to every pairwise statistic, only the two estimators that use the full joint pattern win any cells, the tally on 3 of 15 and the latent normal on 2 of 15, while the pairwise count and the correlation order win none. The pattern of wins therefore matches the form of dependence built into each regime.

\textbf{Generative specification.} A synthetic query is a record with a scalar difficulty and a binary correctness vector over \(m\) simulated models; no text and no language model is involved, and the routing code consumes these records exactly as it consumes benchmark correctness labels. The difficulty is standard normal in four regimes and a fair binary value in the pooled-correlation regime. Each model's conditional correctness probability is a logistic function of \((1,x,x^2-1)\) whose coefficients vary deterministically with the model index, clipped away from zero and one; the higher-order regime instead fixes every marginal at one half. Dependence enters only through how correctness is drawn given these marginals.
\begin{itemize}
\setlength{\itemsep}{1pt}
\item \emph{Conditional independence with mixed difficulty.} Independent draws given \(x\), with per-model intercepts \(1.25-j/(m-1)\), difficulty loadings \(0.65+0.55\,((j \bmod 5)/4)\), and curvature \(0.10\sin(j+0.7)\).
\item \emph{Complementary specialists.} Models form pairs with opposite-sign difficulty loadings \(\pm1.25x\), and each pair shares one uniform draw. The left member is correct when the draw falls below its marginal and the right member when it falls above one minus its marginal, so within-pair failures are as negatively dependent as the marginals permit.
\item \emph{Redundant families.} Blocks of four models share one uniform draw, so family members succeed or fail together as far as their marginals allow.
\item \emph{Higher-order and pairwise-invisible.} Models form triples in which the first two are independent fair bits and the third is their exclusive-or; every pairwise correlation is exactly zero while the triple is deterministic given any two members. Models left over from an incomplete triple are independent fair bits.
\item \emph{Pooled correlation without residual dependence.} The marginals shift by \(0.55\) between the easy and hard difficulty values and draws are independent given difficulty, so pooling over queries creates large raw correlation with no conditional dependence.
\end{itemize}
Each cell uses pool sizes \(m\in\{8,16,32\}\), design budgets \(n_d\in\{50,100,250,500,1000\}\), and twenty seeds, with independent streams of 600 validation, 1{,}000 final-calibration, and 2{,}000 test queries per seed at targets \(\alpha\in\{0.05,0.10\}\). The remaining phase constants are fixed in the released generator.

\begin{table}[t]
\centering
\caption{Synthetic mechanism map, giving the cells out of 15 (three pool sizes, five design budgets) where the ordering beats the design-marginal reference with a bootstrap interval excluding zero, per ground-truth regime.}
\label{tab:mechanism-map}
\scriptsize
\setlength{\tabcolsep}{4pt}
\begin{tabular}{@{}l ccccc@{}}
\toprule
ground-truth regime & tally & pairwise count & error corr. & latent normal & query-local \\
\midrule
conditional independence, mixed difficulty & 0/15 & 0/15 & 0/15 & 0/15 & 0/15 \\
pooled correlation, no residual dependence & 0/15 & 0/15 & 0/15 & 0/15 & 0/15 \\
redundant families (positive residual) & \textbf{10/15} & \textbf{14/15} & \textbf{12/15} & \textbf{10/15} & \textbf{13/15} \\
complementary specialists (negative residual) & 8/15 & 6/15 & 6/15 & 0/15 & \textbf{15/15} \\
higher-order, pairwise-invisible & 3/15 & 0/15 & 0/15 & 2/15 & 0/15 \\
\bottomrule
\end{tabular}
\end{table}

\section{Limitations}
\label{app:limitations}

This section expands the limitations summarized in Section~\ref{sec:conclusion}, stating for each what the present work covers and what would extend it.

\textbf{What the certificate covers.} Theorem~\ref{thm:frozen-feasible-crc} bounds the probability that a certificate is issued and the certified set misses, averaged over the deployment distribution. This is the quantity a routing layer must control, and it is the same quantity that RACER certifies. It is not a per-query guarantee, and it does not by itself certify the answer that a downstream aggregator returns from the called set; the cost case study of Appendix~\ref{app:cost} evaluates one such selector, and certificates for aggregators are a separate line of work. At the tightest targets on the hardest benchmarks, certification is infeasible for every policy alike because the full pool misses too many calibration queries (Section~\ref{sec:experiments}); a larger calibration split or a larger pool, rather than a different policy, moves that boundary.

\textbf{Dependence on the deployment pool and distribution.} The tally order is fitted per benchmark and pool, and the cutoff is calibrated per deployment. Appendix~\ref{app:transfer} shows that an order fitted on one dataset remains useful on a related one at a bounded call premium, and Appendix~\ref{app:cost} shows that adding a model to the pool leaves every certificate intact while removing one requires refitting and recalibrating, at a mean premium of about one call. The operational rule is therefore to refit the order and recalibrate the cutoff whenever the pool or the query distribution changes; both steps run on labels only and add about one call on average after a pool change.

\textbf{Dependence-corrected stopping.} The correction of CopRoute-Dep is estimated from the design queries that the entire prefix misses. These design queries become rare at the depths where a correction would matter, so the estimation noise exceeds what the correction recovers (Table~\ref{tab:dep-ablation}) and CopRoute-Ind is the default. The default reflects the present estimator of the tail, since the order already exploits the dependence information, and a structured estimator of the tail, or a design set several times larger, are the two routes we see to a stopping rule that improves on the independence product.

\textbf{Scope of the theory.} Theorem~\ref{thm:mssc-tight} is a guarantee on the design objective, the mean position of the first correct model under a fixed order, and Appendix~\ref{app:extended-theory} shows when it transfers to fresh queries. It does not bound the number of certified calls at a fixed target, which depends on a quantile of the first-success depth rather than on its mean; the number of calls that CopRoute makes at deployment is therefore an empirical finding, supported across ten datasets, four pool sizes, seven targets, and two routers, while the certificate is the guarantee that holds at test time.

\textbf{Evaluation on recorded labels.} All experiments consume per-query correctness labels released with RouterBench and RouterEval, and no language model is queried. Recorded labels permit 20 redraws across ten datasets and pools of up to 500 models, and they are the protocol under which RACER, FlexRouter, and BaRP are evaluated. It also means that latency, response-dependent cascades, and answer aggregation beyond the router-top selector of Appendix~\ref{app:cost} are outside the evidence presented here, and that the cost case study prices calls on one benchmark with a public price sheet. FlexRouter and BaRP are reconstructions from their papers' specifications, evaluated under the common set-routing protocol rather than their native ones; Appendix~\ref{app:implementation} records every reconstruction choice and Appendix~\ref{app:learned-native} reports their native metrics.

\section{Reproducibility}
\label{app:repro}

The routing evaluations are labels-only computations on public RouterBench and RouterEval releases. Learned baseline fitting uses frozen public query embeddings, while the common evaluation still uses cached correctness matrices and requires no new language-model inference. The public release contains estimator and learned-baseline implementations, split and selection logic, fixed hyperparameter grids, five-seed manifests, per-redraw certification status, standard and augmented diagnostic losses under distinct names, and permutation records. The release includes the aggregated results behind the main comparison (Tables~\ref{tab:main-calls}, \ref{tab:estimator-family}, and~\ref{tab:cost-main}, Figure~\ref{fig:sweep-calls}, and their appendix expansions) together with the scripts that regenerate those tables and figures from those results; the per-redraw artifacts behind the transfer, pool-shift, router-degradation, stopping-rule ablation, and synthetic mechanism-map tables are not included in the release because of their size, and those tables are reproduced by rerunning the corresponding scripts on the public benchmark data.

\end{document}